\documentclass[10pt]{article}

\usepackage[cmex10]{amsmath}
\usepackage{amssymb, amsfonts, amsthm}
\usepackage{mathptmx}

\usepackage{graphicx}
\graphicspath{{./figures/},{./poster/}}
\DeclareGraphicsExtensions{.pdf,.jpeg,.png}

\usepackage{grffile}
\usepackage{hyperref}
\usepackage{paralist}
\usepackage{array}
\usepackage{natbib}
\usepackage[margin=1.75in]{geometry}

\theoremstyle{definition}
\newtheorem{definition}{Definition}[section]
\newtheorem{theorem}{Theorem}[section]

\title{Saddlepoints in Unsupervised Least Squares}

\author{Samuel Gerber }

\begin{document}

\maketitle

\begin{abstract}
This paper sheds light on the risk landscape of unsupervised least squares in
the context of deep auto--encoding neural nets. We formally establish an
equivalence between unsupervised least squares and principal manifolds. This
link provides insight into the risk landscape of auto--encoding under the mean
squared error, in particular all non-trivial critical points are saddlepoints.
Finding saddlepoints is in itself difficult, overcomplete auto--encoding poses
the additional challenge that the saddlepoints are degenerate.  Within this
context we discuss regularization of auto--encoders, in particular bottleneck,
denoising and contraction auto--encoding and propose a new optimization
strategy that can be framed as particular form of contractive regularization.
\end{abstract}

\section{Introduction}
\label{sec:introduction}
This papers examines the risk landscape of unsupervised least squares.
Unsupervised least squares aims to find a reconstruction map $r$ that
minimizes the discrepancy of the output to it's own input , i.e., minimizing
$\lVert r(x) - x\rVert^2$. In the neural network literature this is referred to
as auto--encoding. Without any restrictions on $r$ the identity mapping is an
optimal, albeit uninformative, solution. An informative solution should provide
a salient summary of the data. To find informative solutions, auto--encoding is
typically formulated as a concatenation of an encoder map $\lambda$ and a
decoder map $g$ with $r = g \circ \lambda$ with restrictions on
$\lambda$ or $g$ to avoid learning the identity function. Several
regularization methods exist to enforce such restrictions. For example, the
bottleneck auto--encoder forces the map $\lambda$ to be of lower
dimension than the input through adding a bottleneck layer.  While such
regularization avoid learning the identity map, they do not address the general
problem of overfitting in unsupervised least squares.

For networks with linear activations, \citep{kawaguchi2016deep} show that for
auto-encoding  all minima are globally optimal and the solutions correpsonds to
the maximal principal subspace. This paper provides an analogous
characterization of the critical points of the risk function for deep
auto-encoders in the non-linear case.  In particular all critical points,
besides minima with zero risk, are saddle points. Zero risk solutions defy the
purpose of autencoders, i.e.  finding a more compact representation that
characterizes the input data. In the linear case solution with zero risk are
identity functions. Zero In the non-linear case zero risk solutions include,
besides the identity, space filling manifolds. This highlights an important
distinction between the linear and non-linear case. In the linear case it is
sufficent to restrict the dimensionality of the principal subspace, for example
through the use of a bottleneck layer. For non-linear auto-encoders additional
regularization of the encoding and decoding functions is required, otherwise
the optimzation will tend towards space-filling solutions.

We establish a tight connection between unsupervised least squares and the more
restrictive notion of a principal manifold. This connection illuminates the
risk landscape of unsupervised leas squares, which in turn informs the
behaviour of auto--encoding under various forms of regularization.
%Principal manifolds, are an approach to manifold
%learning that construct an explicit manifold and fit it to the data by
%minimizing projection error.
\citet{hastie:jasa89} formally defined principal curves as curves that pass
through the middle of a distribution and showed that principal curves are
critical points of the mean squared reconstruction error.  This indicates a
tight connection to unsupervised least squares. Principal manifolds enforce
that the reconstruction function constitutes an orthogonal projection, while
unsupervised least squares permit arbitrary reconstruction functions. To
establish the equivalence between principal manifolds and unsupervised least
squares we show that the critical points of the mean squared error risk for
arbitrary reconstruction functions, and thus the critical points of
auto--encoders under squared error loss, are principal manifolds. The
connection of unsupervised least squares to principal manifolds illuminates the
shape of the auto--encoder risk landscape.  In particular \citet{duchamp:as96}
showed that principal curves are saddlepoints of the mean squared error, this
observation also holds for auto--encoding.  \citet{duchamp:as96} and
\citet{gerber2013regularization} demonstrate that the saddlepoint nature of the
risk landscape leads to overfitting and renders traditional cross-validation
techniques infeasible for tuning of regularization parameters.

This difficulty persists for auto--encoders as Figure~\ref{fig:hard-constraint}
illustrates in the case of bottleneck auto--encoding. The bottleneck layer
enforces a hard constraint on the dimensionality of the reconstruction
function, while the network architecture implicitly regularizes the complexity
of the reconstruction function.  However, due to the saddlepoint nature of the
risk landscape the risk on test data keeps decreasing with increasing network
complexity. Thus, cross-validation will fail to suggest an appropriate network
complexity. In the limit, with infinite training data and no restriction on the
complexity of the reconstruction function a space filling curve, with zero
reconstruction error, is an optimal fit. These global minima, space-filling
manifolds or the identity function, of the auto--encoding risk are not
desirable solutions and do not provide an adequate summary representation of
the data.
\begin{figure}[htb]
\centering
\begin{tabular}{ccc}
  \includegraphics[width=0.3\linewidth]{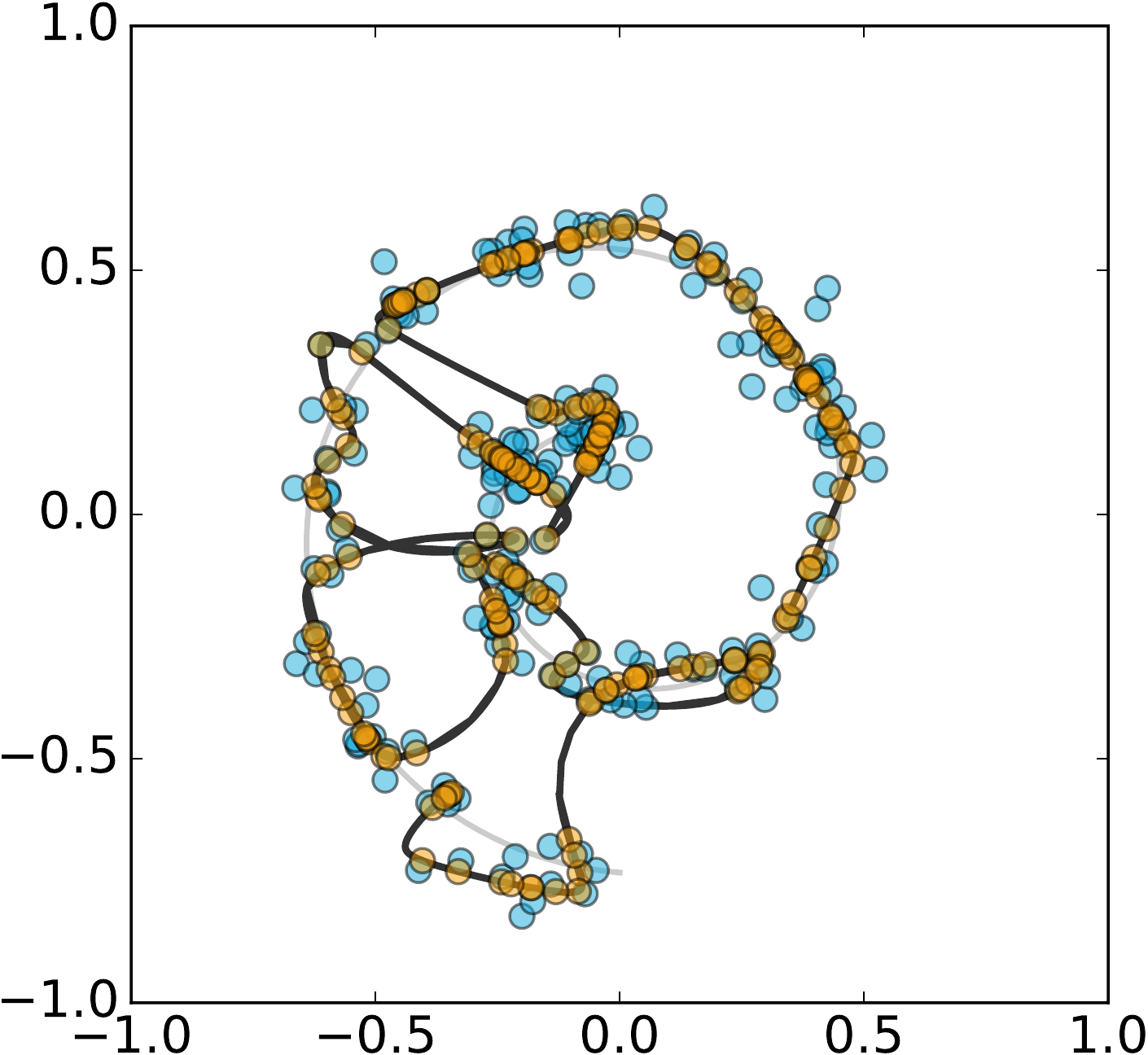} &
  \includegraphics[width=0.3\linewidth]{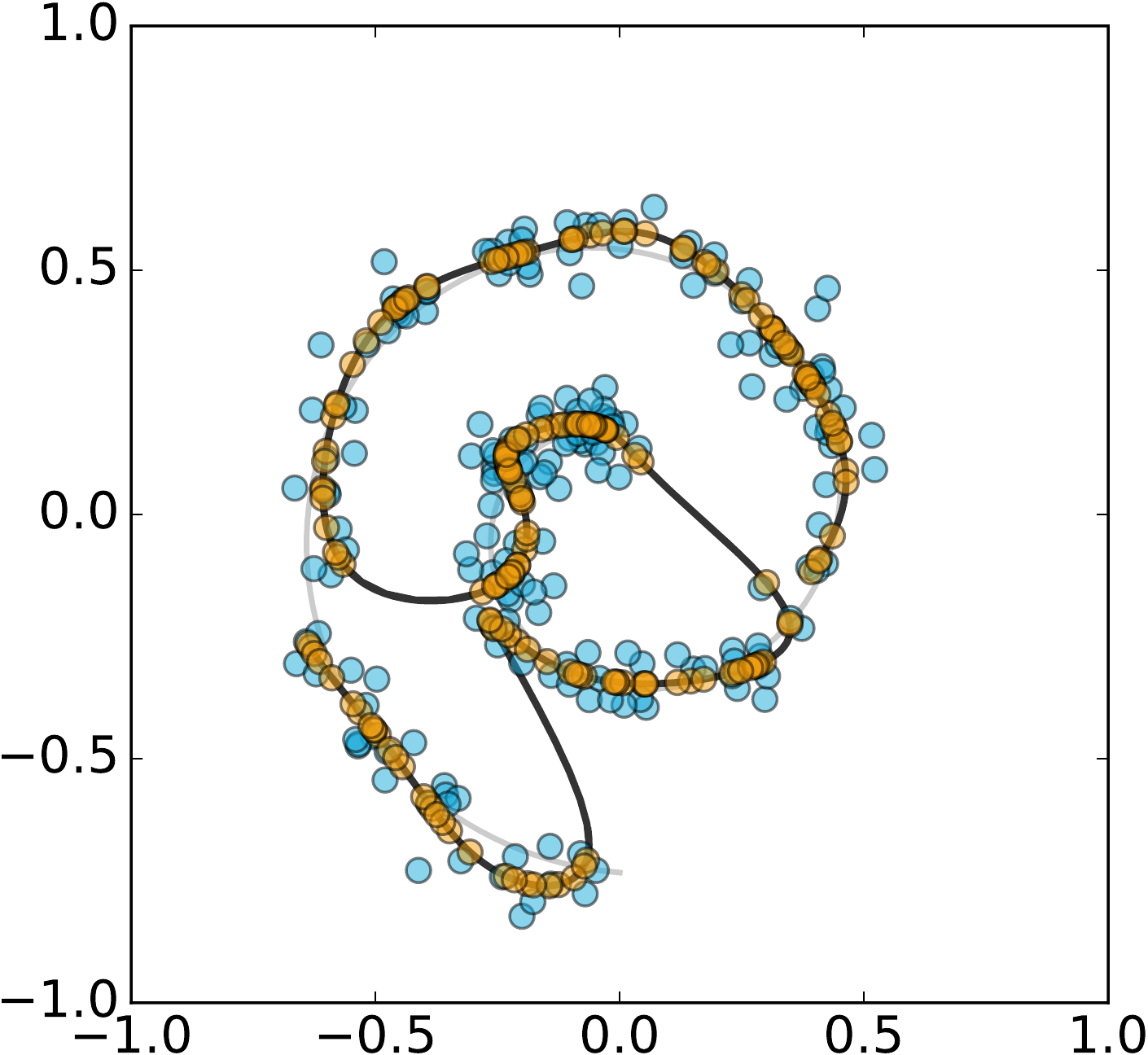} &
  \includegraphics[width=0.3\linewidth]{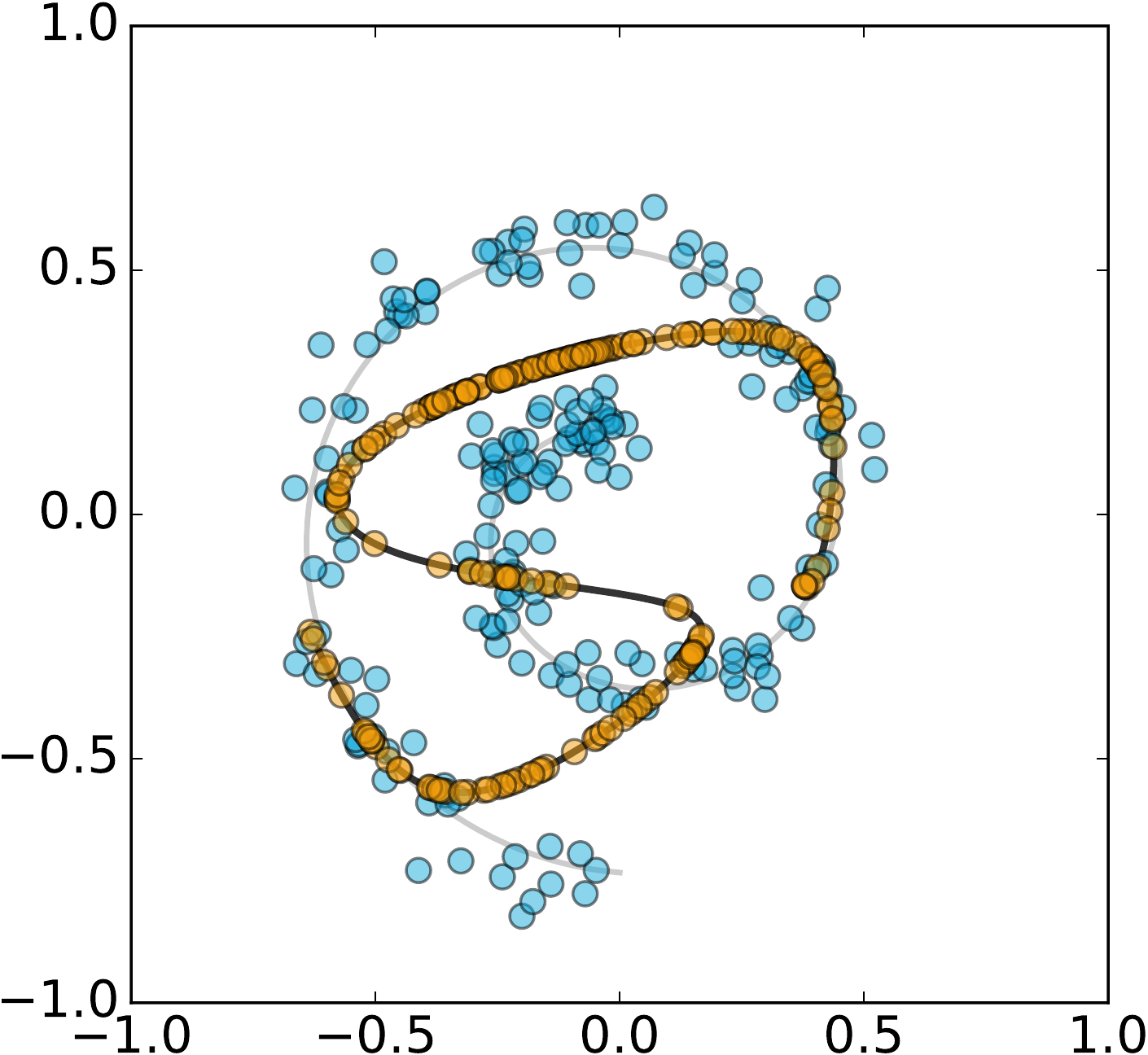} \\
  \includegraphics[width=0.3\linewidth]{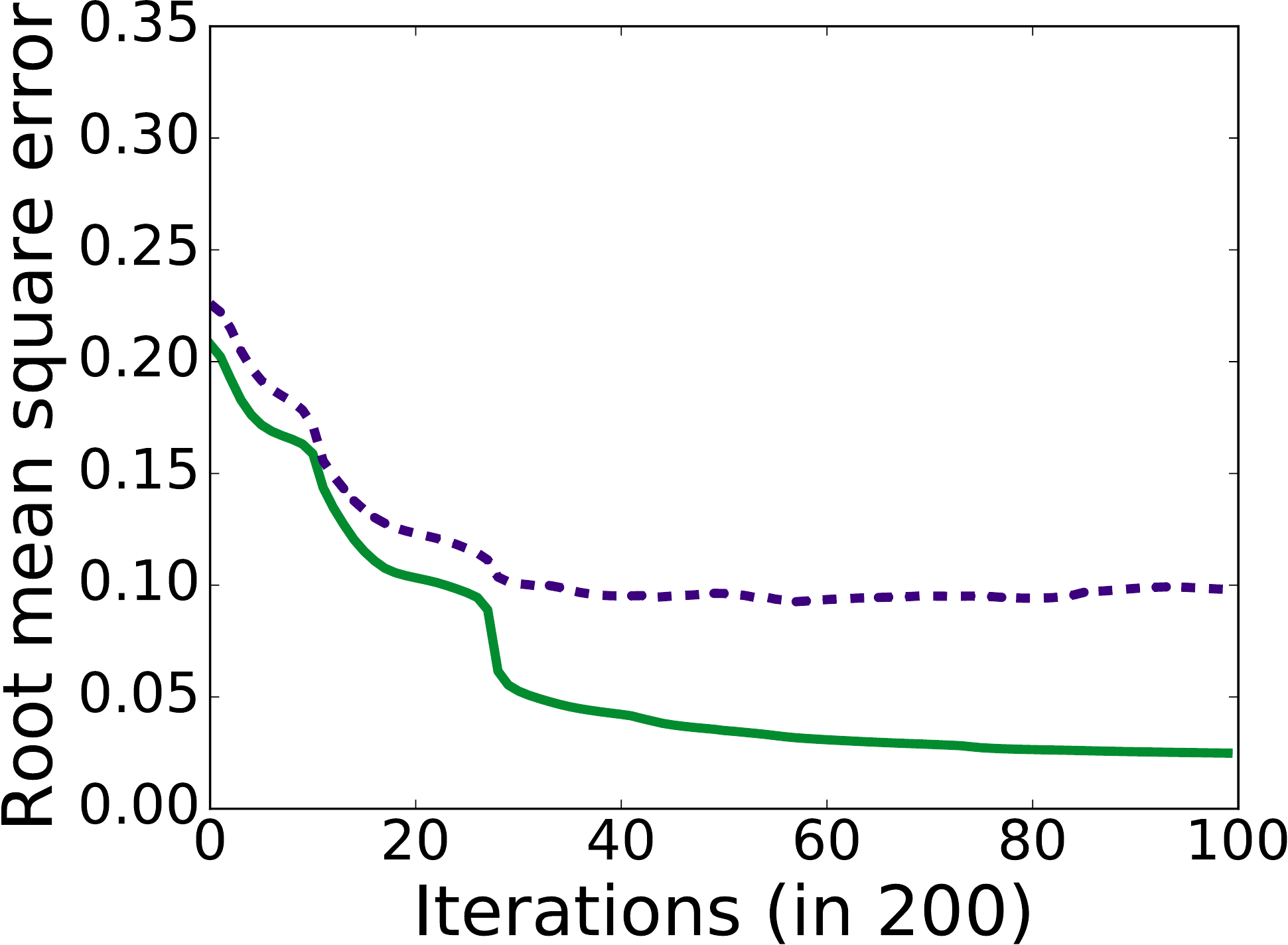} &
  \includegraphics[width=0.3\linewidth]{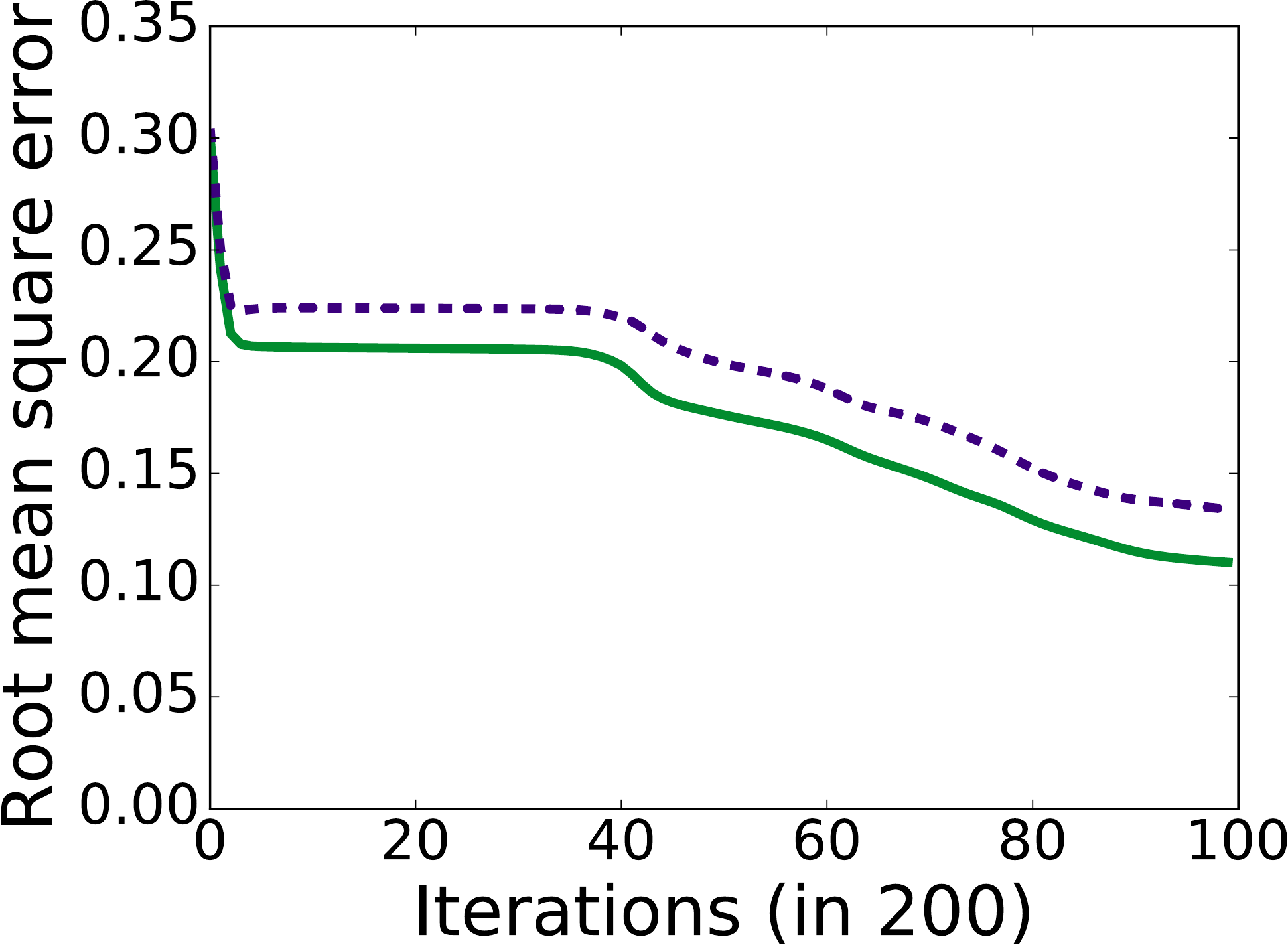} &
  \includegraphics[width=0.3\linewidth]{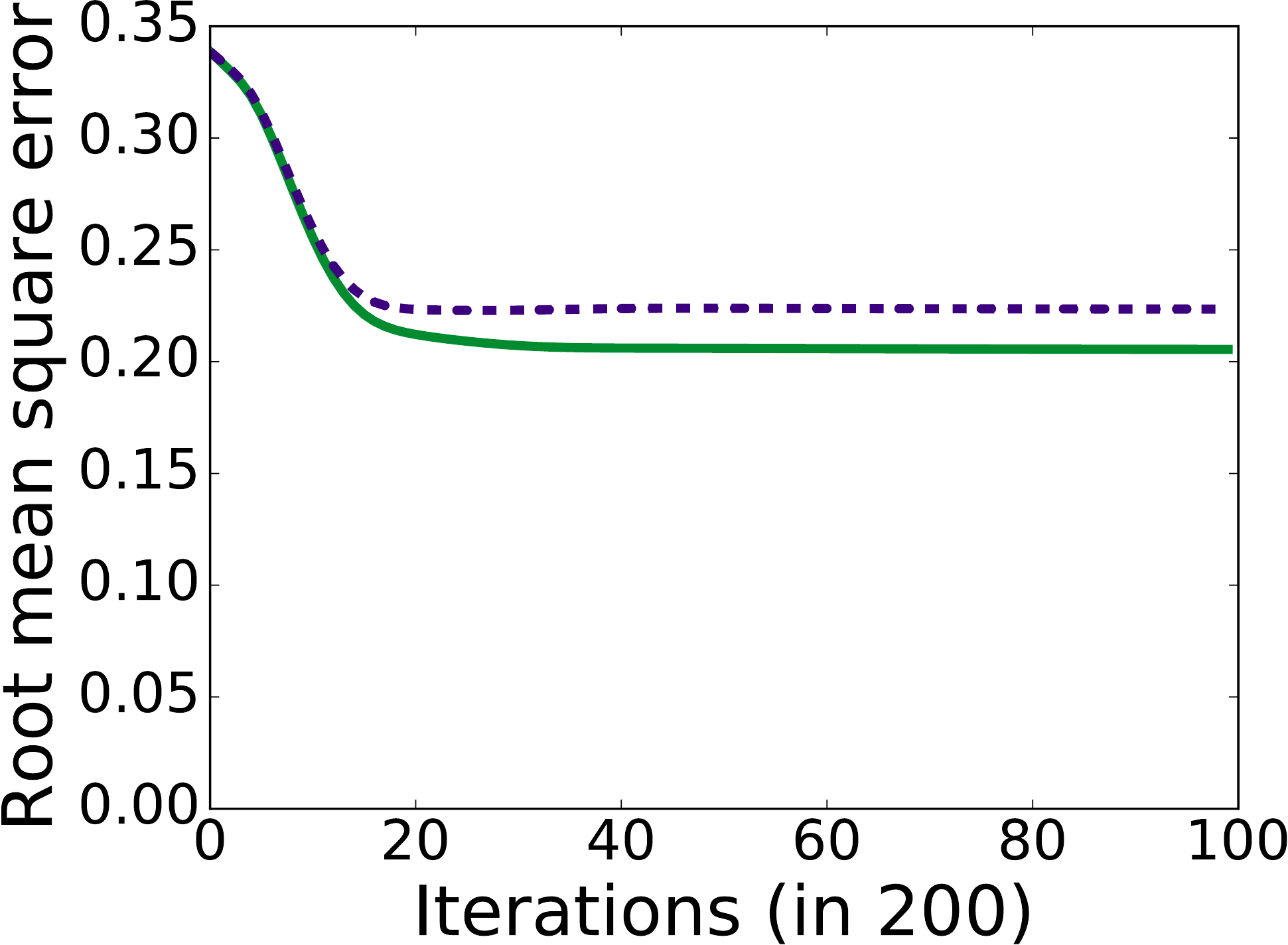} \\
  (a) & (b) & (c)
\end{tabular}
\caption{Bottleneck auto--encoding enforces a hard dimensionality constrain
  through a  bottleneck neural network layer.  The images show (top row) the
  minima found after 20,000 iterations and (bottom row) the root mean square
  error on training (green solid line) and test data (purple dashed
  line)  using network architectures with hidden layers of (a)
  50--100--200--100--50--1--50--100-200--100--50, (b)
  50--100--50--1--50--100--50, (c) 200--1--200  units. Depending on the network
  architecture, auto--encoding does (a) overfit,  (b) fit  well or (c)
  underfit. The relatively small reach of the spiral example requires a fairly
  deep network to achieve an accurate fit. However, the test error does not
  suggest that the fit in (b) is adequate and suggests to use the most powerful
  network. In all cases the one-dimensional parametrization of the
  auto--encoder jumps across the spiral arms.
  \label{fig:hard-constraint}
}
\vspace{-0.05in}
\end{figure}
The hard dimensionality constraint in bottleneck auto--encoders avoids fitting
the identity function. However, the hard constraint requires careful
initialization to find an adequate parametrization.  As an alternative to the
bottleneck layer several regularization schemes, such as
denoising~\citep{vincent2008extracting} and contractive
auto--encoding\citep{alain2014regularized}, for controlling the dimensionality
of the reconstruction function have been proposed. These regularization methods
provide a soft constraints on dimensionality and solve or circumvent the issue
of inadequate parametrization. The soft dimensionality constraints do not
alleviate the challenge of tuning the regularization or selecting an
appropriate network architecture. Figure~\ref{fig:contractive-architecture}
shows that for a fixed contractive penalty the test error for different network
architecture either selects an architecture that severely overfits or tends
towards the identity mapping.
\begin{figure}[htb]
\centering
\begin{tabular}{ccc}
  \includegraphics[width=0.3\linewidth]{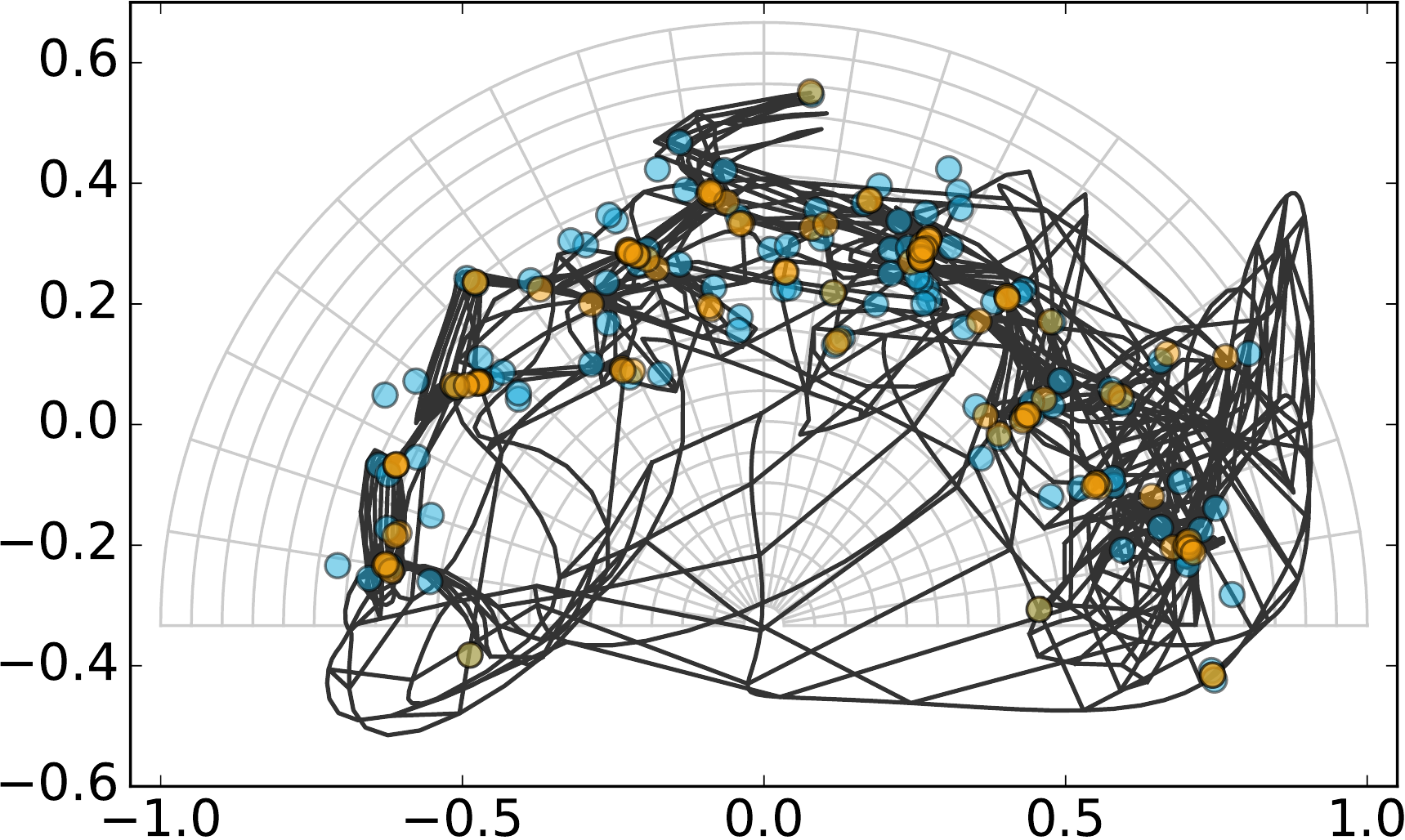} &
  \includegraphics[width=0.3\linewidth]{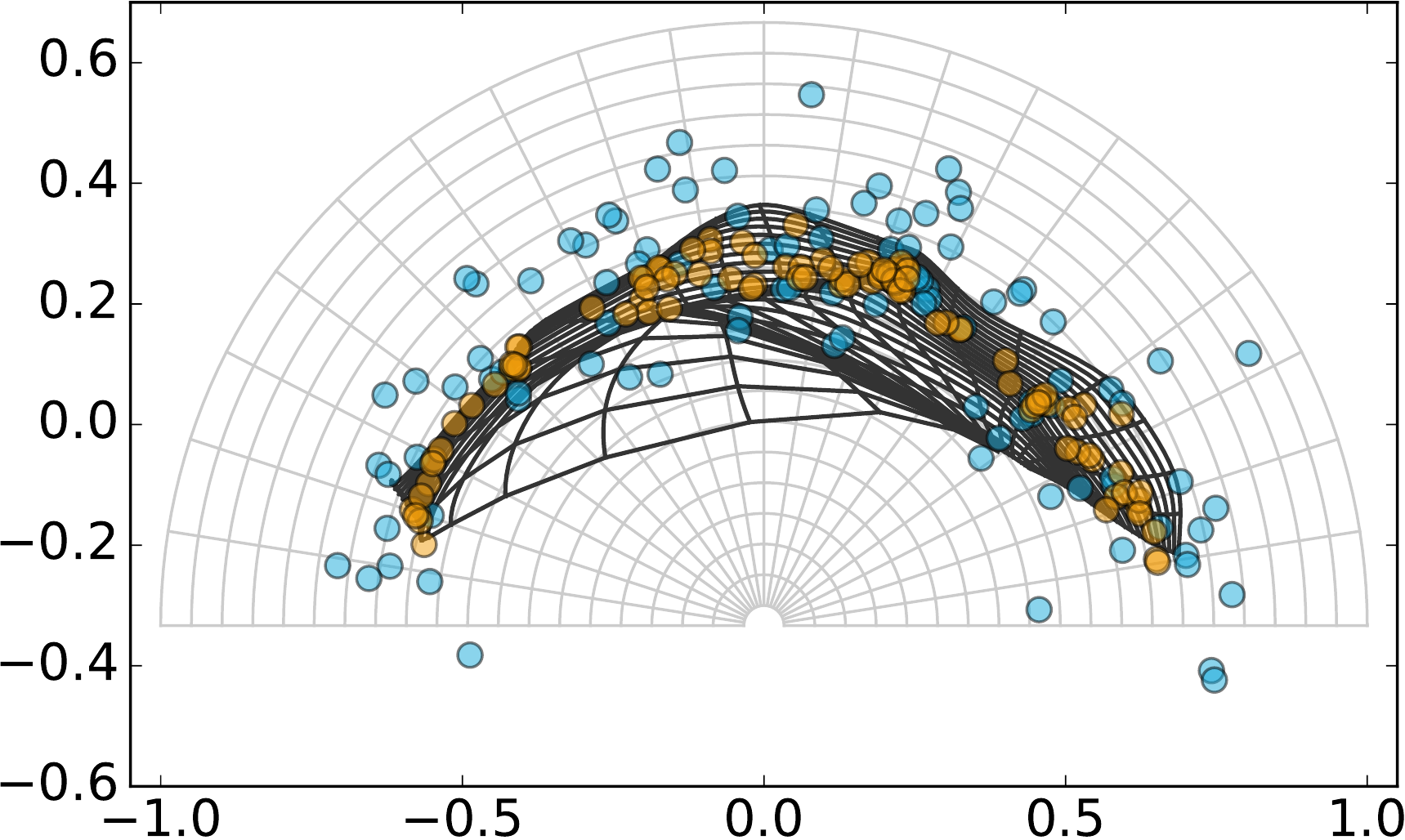} &
  \includegraphics[width=0.3\linewidth]{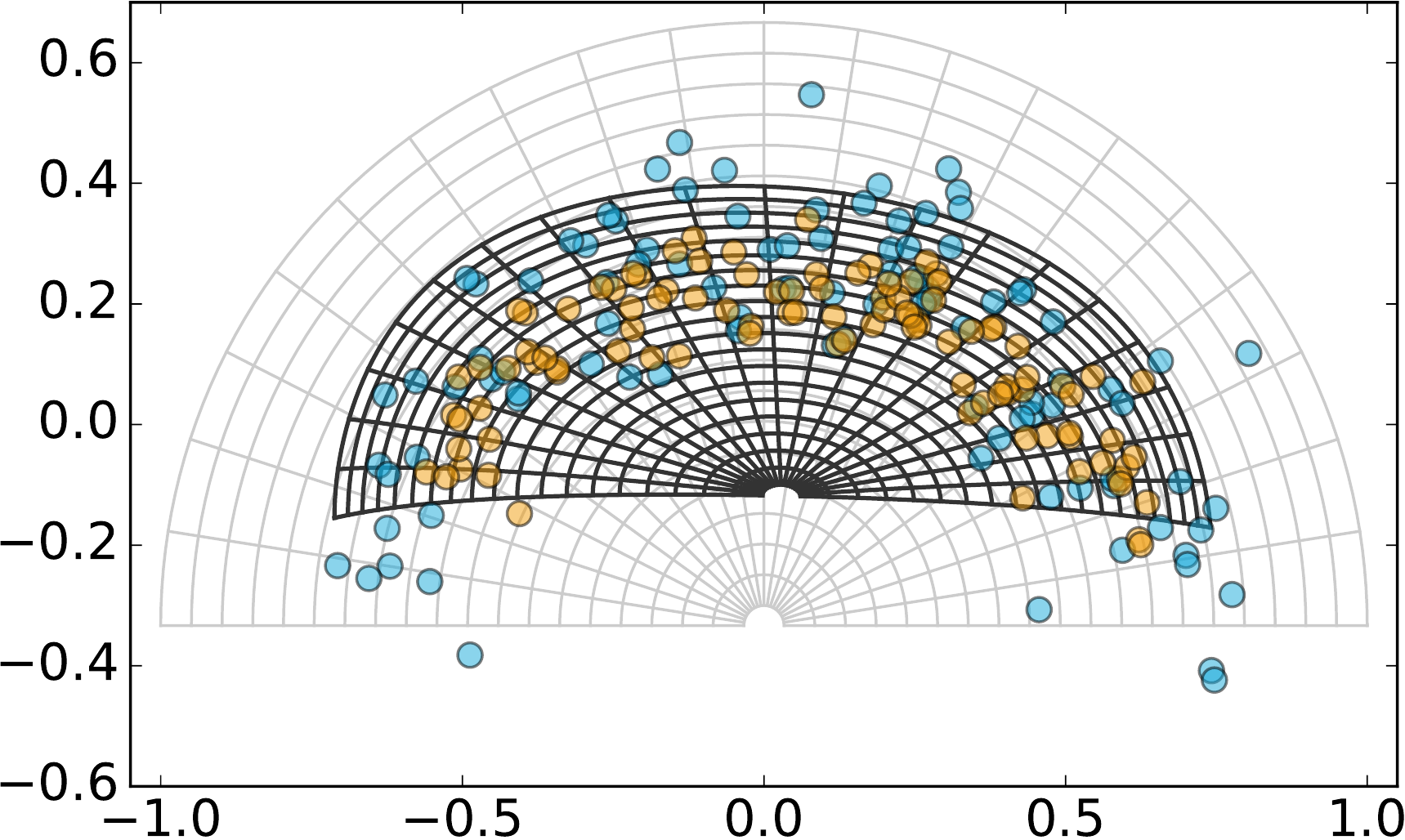} \\
  \includegraphics[width=0.3\linewidth]{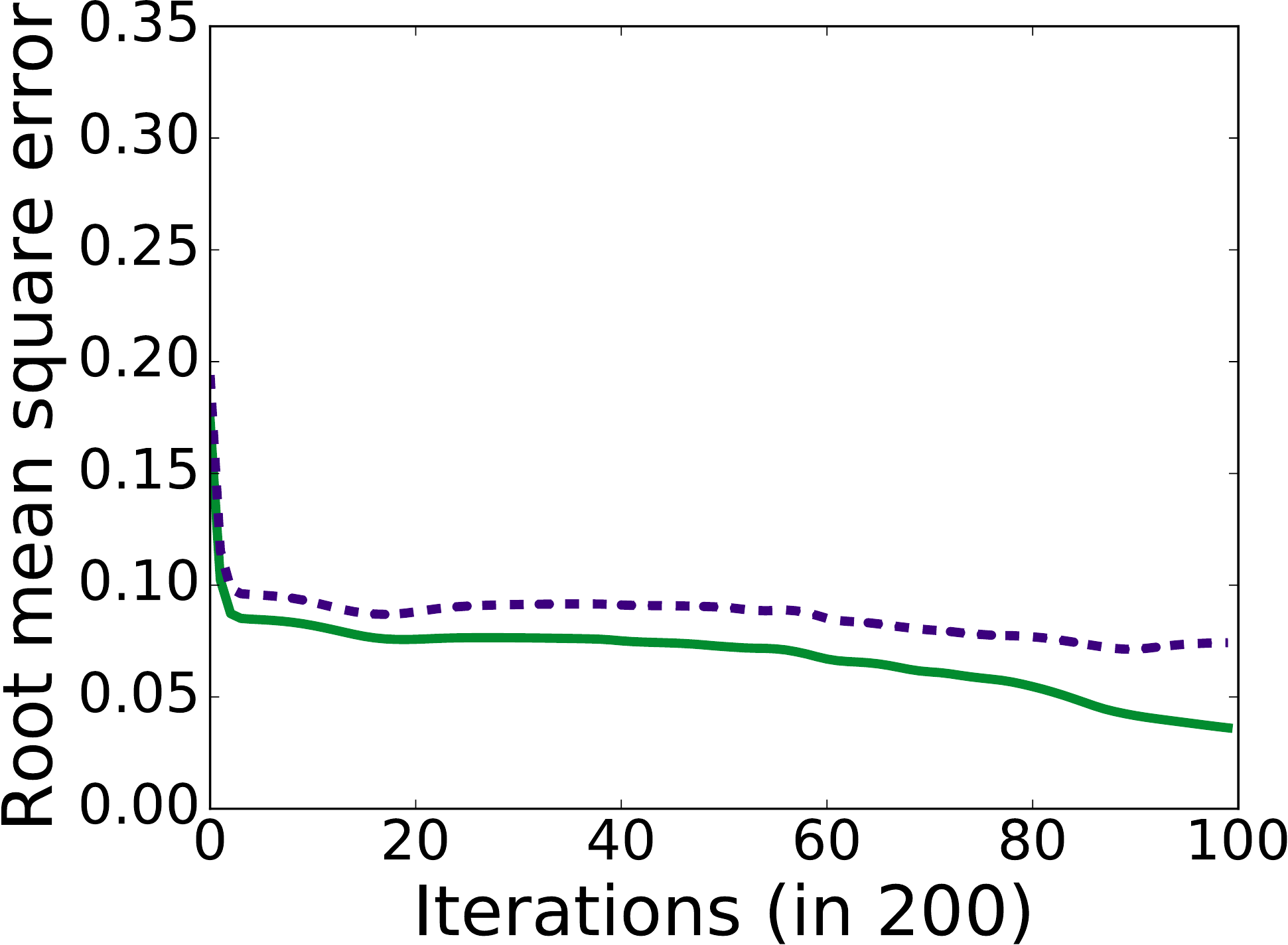} &
  \includegraphics[width=0.3\linewidth]{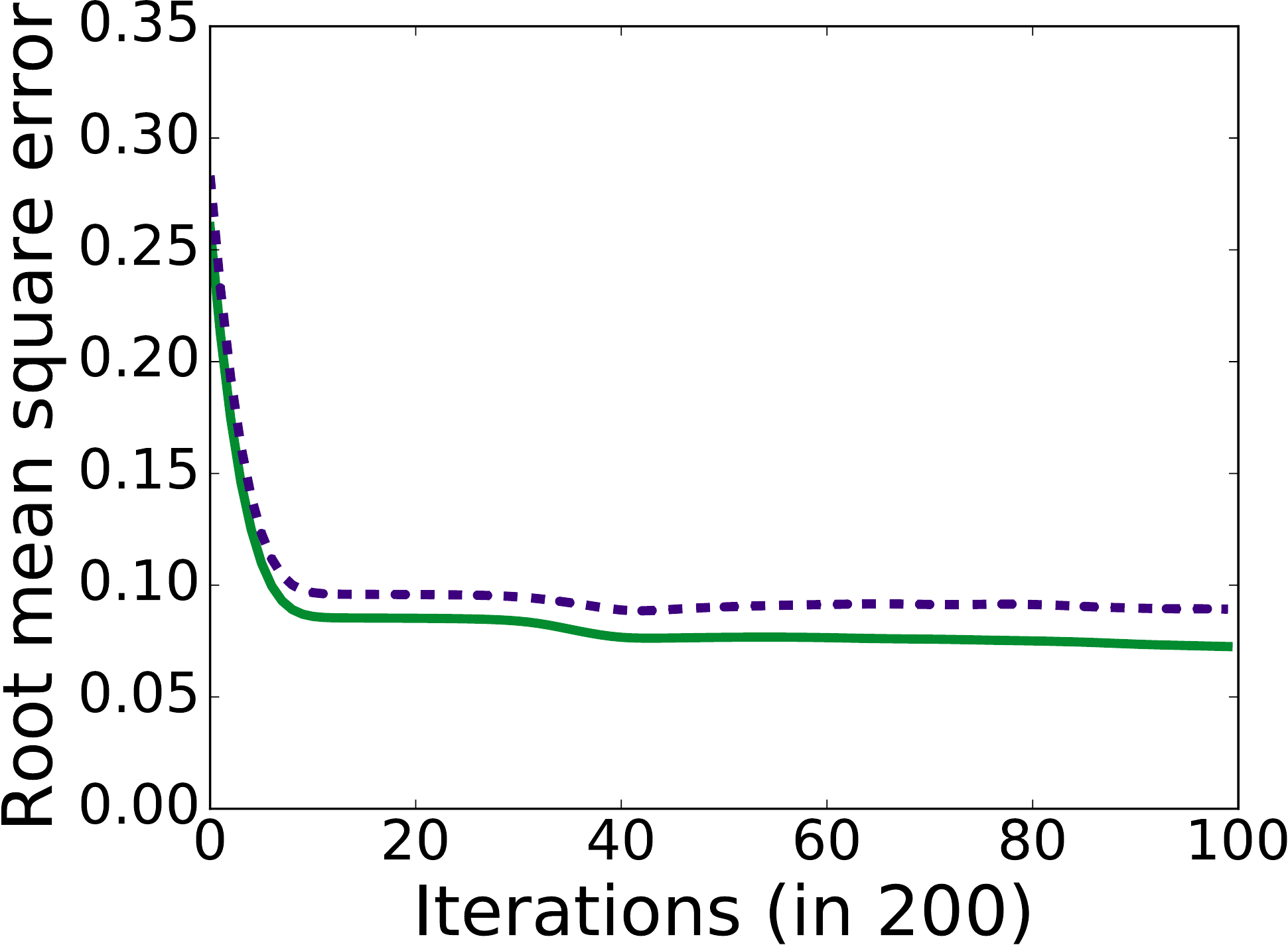} &
  \includegraphics[width=0.3\linewidth]{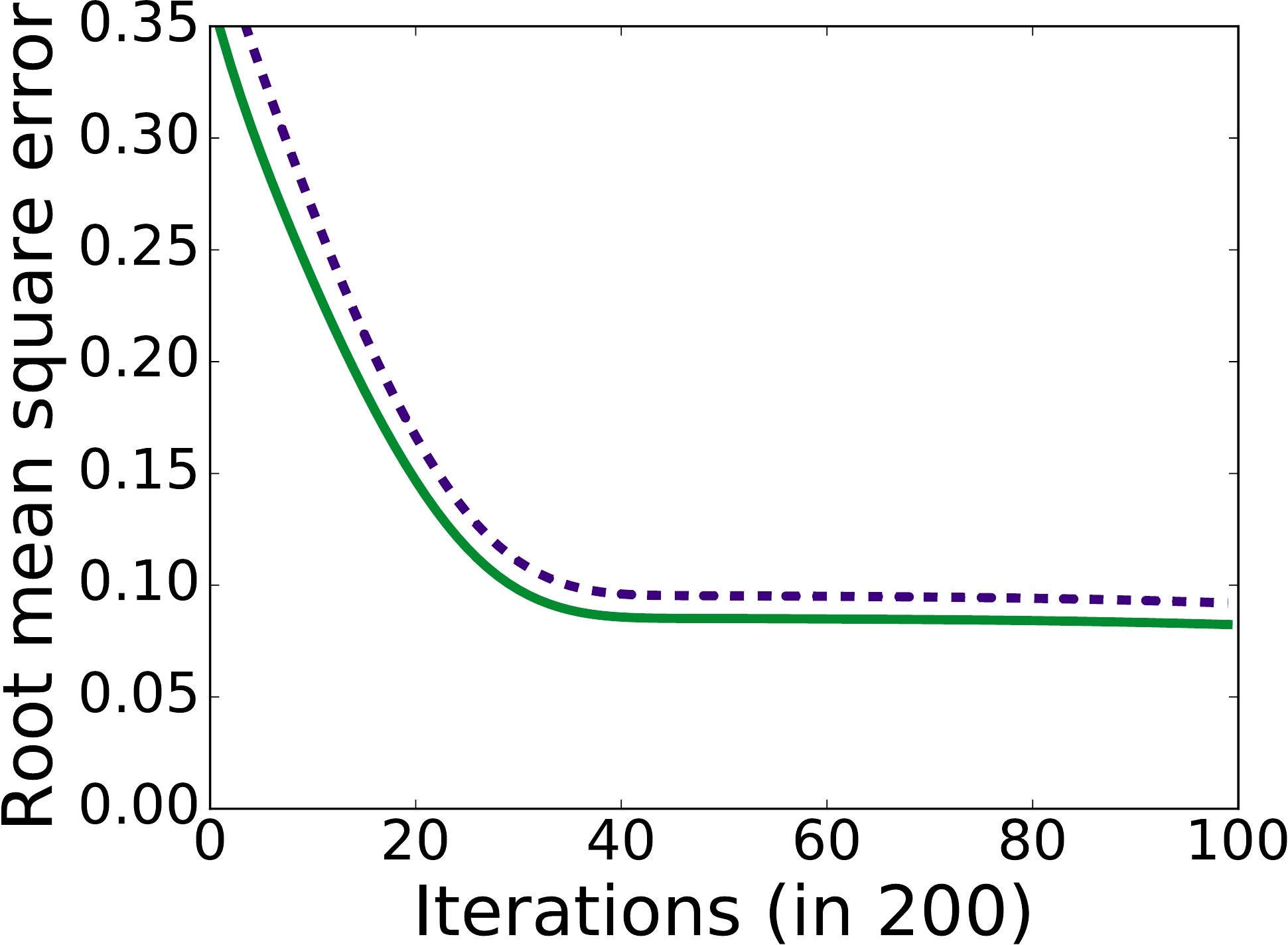} \\

  (a) & (b) & (c)
\end{tabular}
\caption{Effect of network architecture in combination with contractive
  regularization penalty. The image show the (top row) minima found after 20,000
  iterations and (bottom row) the root mean square
  error on training (green solid line) and test data (purple dashed
  line)  using a fixed penalty of 0.02 and network architectures with
  hidden layers of (a) 50--100--200--200--100--50, (b) 50--100--100--50, (c)
  200--200 units. The gray lines show a polar grid on the input space and the
  black lines show the deformation of the polar grid after mapping it through
  the auto--encoder.  The gray line shows the spiral the data is sample from
  with noise.  The black lines shows the spiral after mapping it through the
  auto--encoder.
  \label{fig:contractive-architecture}
}
\vspace{-0.05in}
\end{figure}
Fixing the network architecture while adjusting the regularization does not
solve the issue as Figure~\ref{fig:contractive-penalty} illustrates. A very
flexible network that permits severe overfitting leads to the smallest test
error. If the network architecture is restrictive enough the test error
decreases with decreasing penalty weight and the solution tends towards the
identity mapping.
\begin{figure}[htb]
\centering
\begin{tabular}{ccc}
  \includegraphics[width=0.3\linewidth]{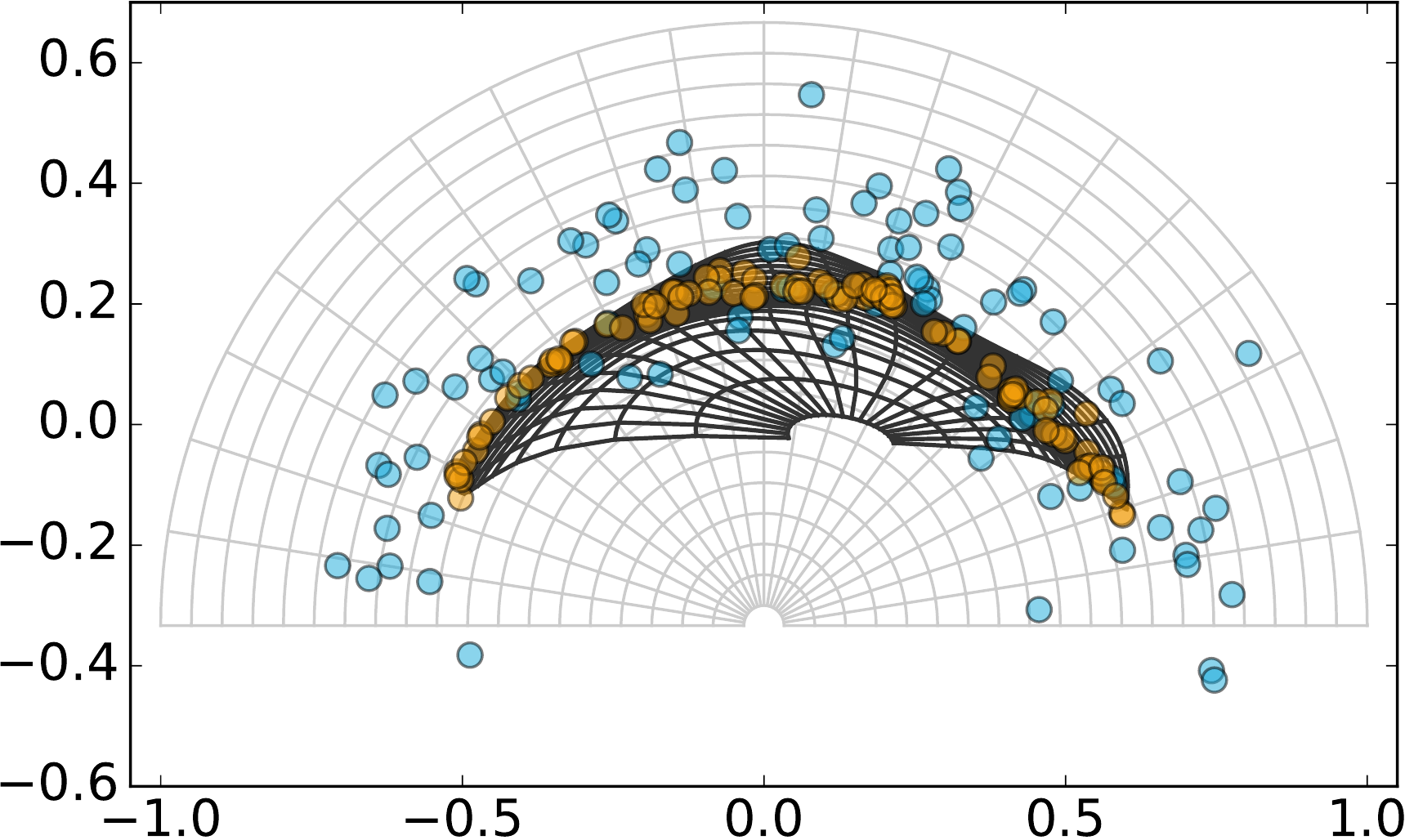} &
  \includegraphics[width=0.3\linewidth]{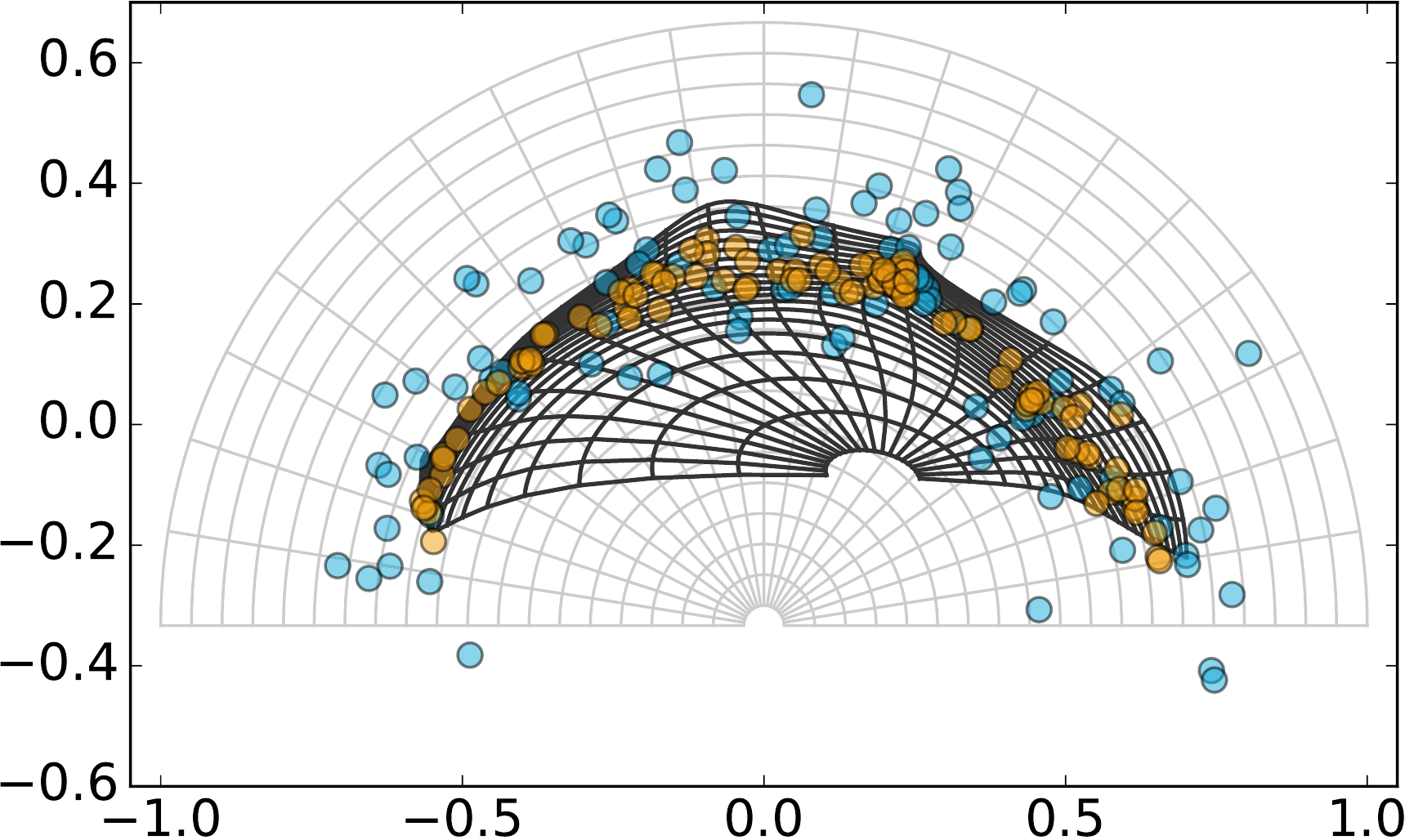} &
  \includegraphics[width=0.3\linewidth]{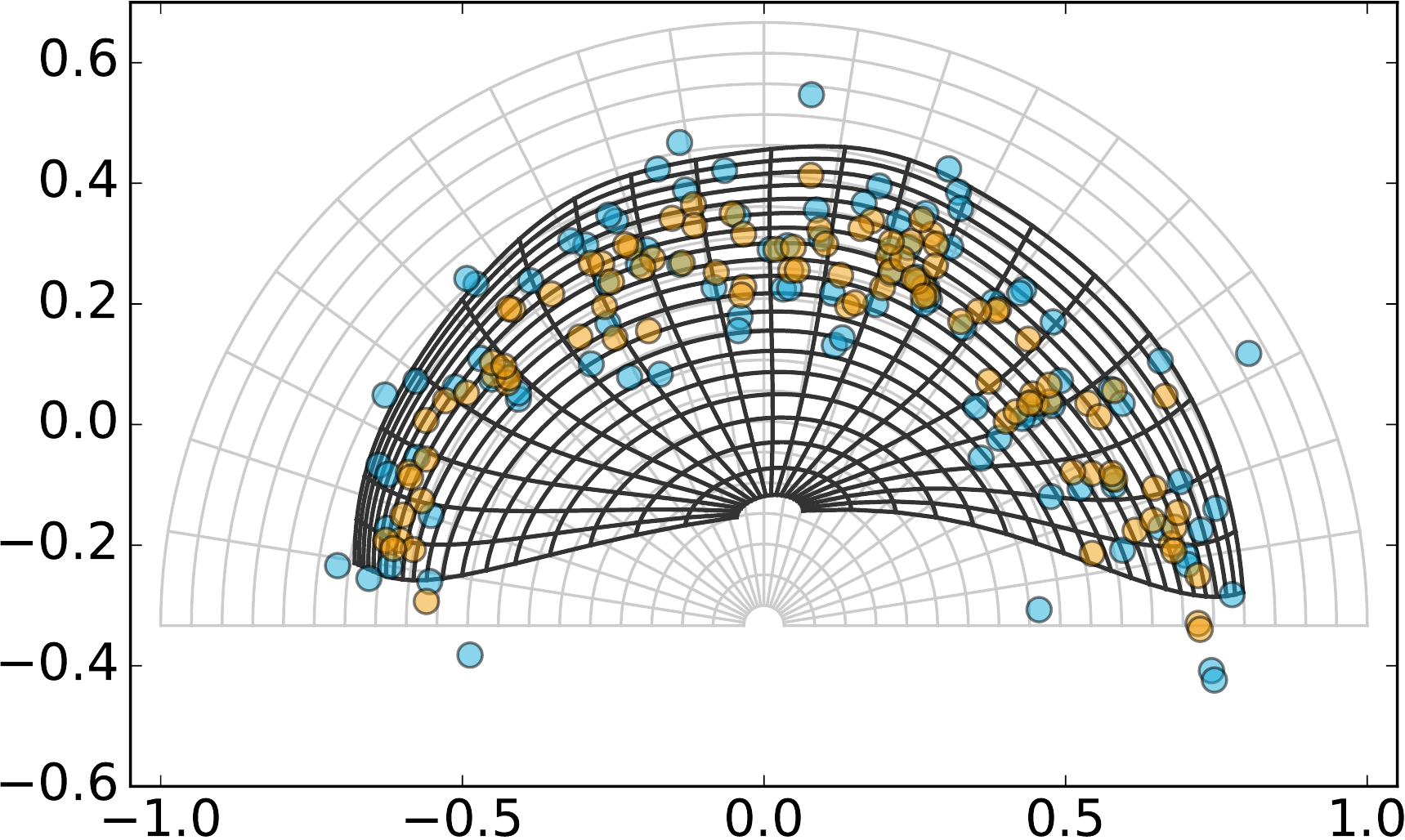} \\
  \includegraphics[width=0.3\linewidth]{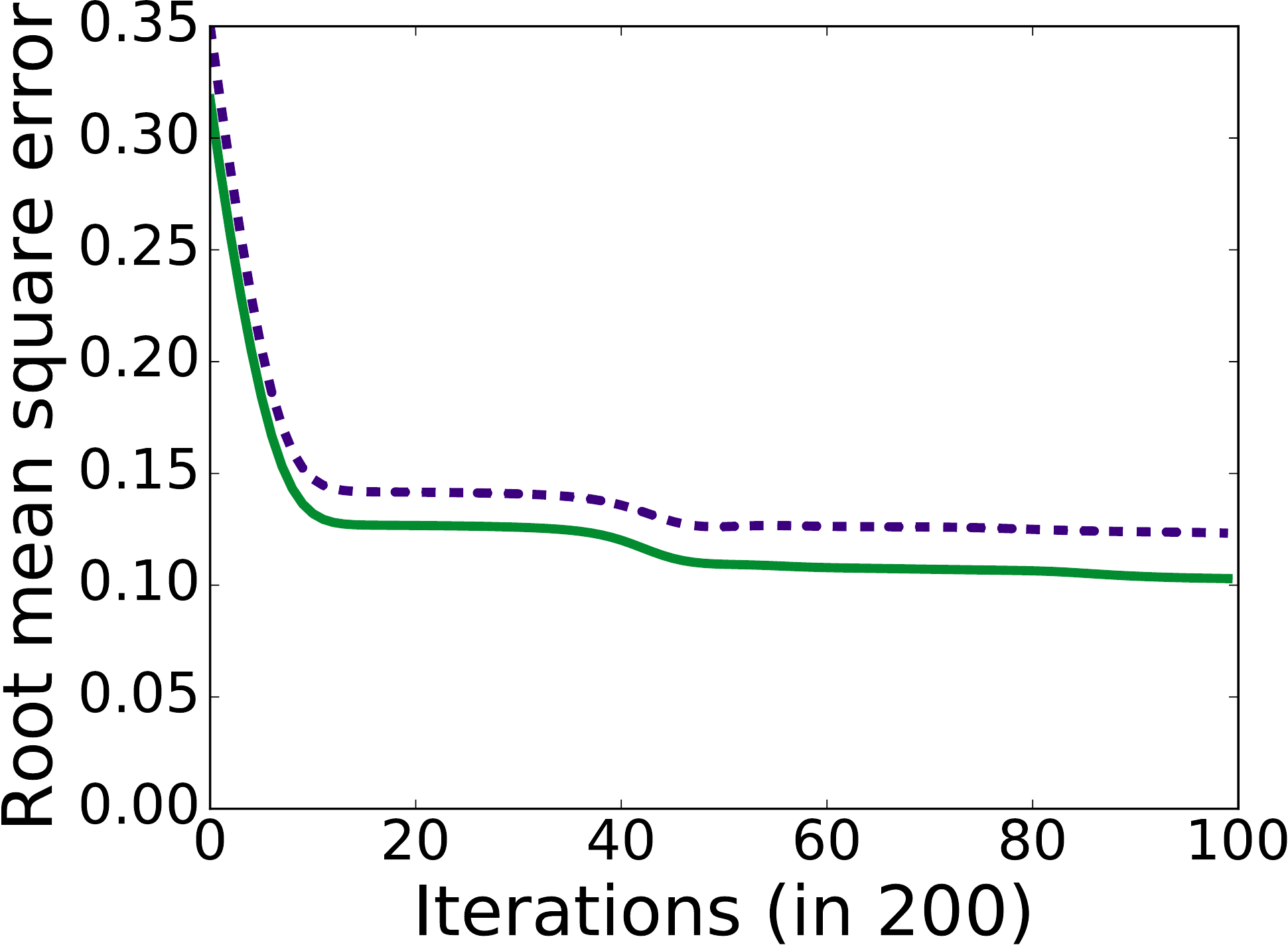} &
  \includegraphics[width=0.3\linewidth]{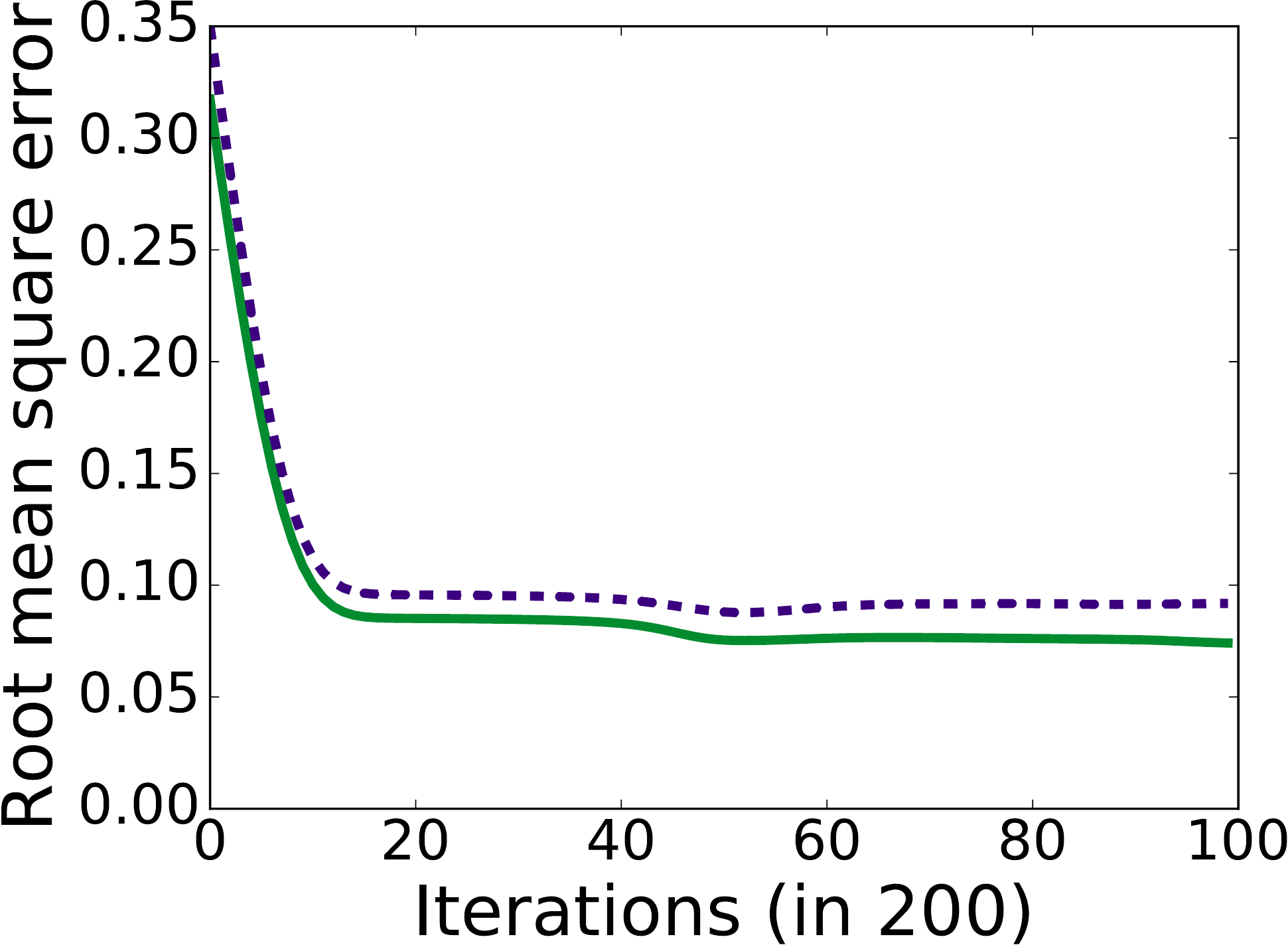} &
  \includegraphics[width=0.3\linewidth]{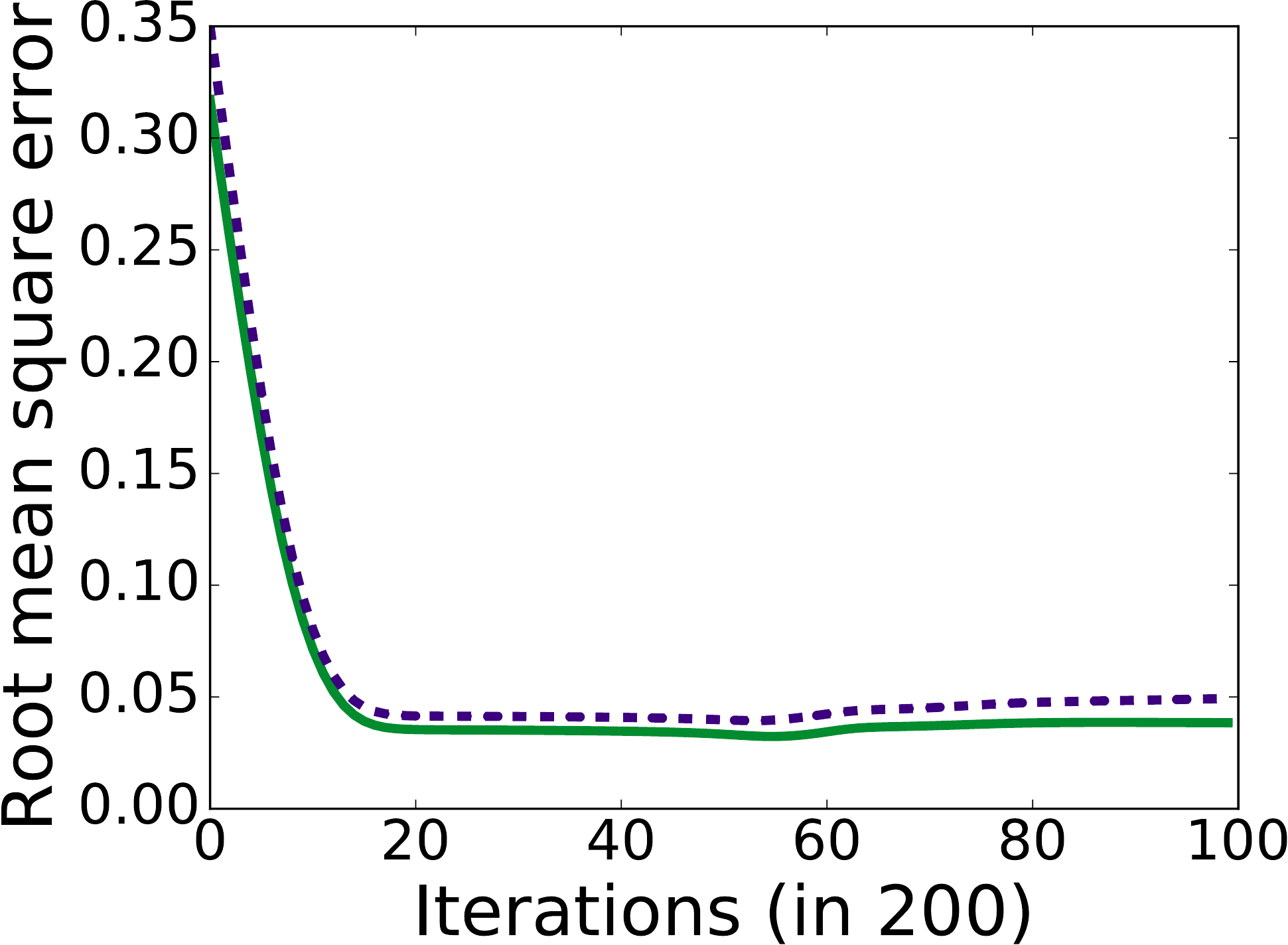} \\
  (a) & (b) & (c)
\end{tabular}
\caption{Effect of regularization penalty for fixed network architecture. The
  image show the (top row) minima found after 20,000 iterations and (bottom
  row) the root mean square error on training (green solid line) and test data
  (purple dashed line)  using a network architecture with hidden layers of
  50--50--50--50 units and  regularization weights (a) 0.04 (b)  0.02 and (c)
  0.005. The gray lines show a polar grid on the input space and the
  black lines show the deformation of the polar grid after mapping it through
  the auto--encoder.
  \label{fig:contractive-penalty}
}
\vspace{-0.05in}
\end{figure}
In practice the structure in the data is typically not know and it would be
difficult to select an appropriate network architecture. Thus, a method that
imposes a dimensionality constraint without reliance on the implicit
regularization of the network architecture is desirable. We contend that the
difficulties in regularization of auto--encoders are due to the saddlepoint
structure of the unsupervised least squares risk.

The critical points of the unsupervised least squares risk are principal
manifolds and provide an arguably reasonable summary representation of the
data. However, except for the uninformative identity solution and space filling
manifolds the critical points are saddlepoints. Figures
Figures~\ref{fig:hard-constraint},~\ref{fig:contractive-architecture}
and~\ref{fig:contractive-penalty} illustrate that the saddlepoint nature of the
critical points renders cross--validation useless, since it is possible to move
away from a critical point without increasing the risk on test data.
Furthermore steepest decent methods will typically not even move towards
desirable solutions.  In the principal manifold setting,
\citet{gerber2013regularization} propose a solution to the saddlepoint
challenge based on minimizing a different risk. The novel risk has the property
that principal manifolds are now minima instead of saddlepoints. We show that
this risk, derived from geometrical considerations, is a particular application
of a more general approach based on minimizing the norm of the Gradient of the
risk.  We apply this {\em Gradient--Norm minimization } to auto--encoders,
which results in a formulation that bears a close resemblance to contractive
auto--encoding~\citep{alain2014regularized}.

%We display the difficulties in regularizing auto--encoders in the context
%of bottleneck, denoising and contractive auto--encoders and demonstrate the
%proposed method for finding saddlepoints. The treatment is mostly theoretical
%and the effects are illustrated on toy examples.

%The main contributions of this paper paper are:
%\begin{itemize}
%\item Establishing a formal connection of auto--encoding to principal
%  manifolds, which shows that critical points of the auto--encoder risk are
%  saddlepoints.
%\item Providing an in depth examination of existing different regularization
%  schemes and their effect on the risk.
%\item Developing a novel optimization scheme---Gradient--Norm minimzation---for
%  finding saddlepoints and establishing the connection to contractive
%  auto--encoding as well as to a recent result on principal manifolds.
%\end{itemize}
%It is easy to
%construct examples in which the proposed method outperforms other
%regularization methods if the aim is to construct principal manifolds, e.g.,
%reconstruction of noisy data points as in~\citet{gerber2013regularization}. Due
%to dubious additional informational value of such experiments, we exclude them
%from this manuscript.

\section{The Unsupervised Least Squares Risk}
\label{sec:risk}
This section shows that the critical points of the unsupervised least squares
risk are principal manifolds and investigates the properties of those critical
points in detail. Section~\ref{sec:bgpc} revisits the definition of a principal
curves and manifolds and Section~\ref{sec:critical} establishes the connection
of the unsupervised least squares risk to principal manifolds by examination of
the first variation of the unsupervised least squares risk.

\subsection{Connection to Principal Manifolds}
\label{sec:bgpc}
Let $X$ be a random variable with a smooth density
$p$ such that the support $\Omega =\{ x:p(x) > 0\}$ is a compact, connected
region with smooth boundary.  Denote by $E$ the expectation operator, i.e.,
$E\left[ f(X)  \right] = \int_{\Omega} f(x) p(x) dx $.

Recall the formal definition of principal curves.
\begin{definition}[Principal
Curve~\citep{hastie:jasa89}] Let $g:\Lambda \to \mathbf{R}^n$, $\Lambda
\subset \mathbf{R}$ and $\lambda_g:\Omega \to \Lambda$ with {\em projection
index} $\lambda_g(x) = \max_{z} \{ s : \lVert y - g(x) \rVert =
\inf_{\tilde{z}} \lVert y - g(\tilde{s})\rVert \}$. The principal curves of $Y$
are the set $\mathcal{G}$ of smooth functions $g$ that fulfill the self
consistency property $E[Y \vert \lambda_g(X) = z] = g(z)$.
\end{definition}
\citet{hastie:jasa89} showed that principal curves are critical points of the
risk $d(g, X)^2 = \frac{1}{2}  E[ \lVert g(\lambda_g(X)) - X \rVert^2 ]$. The
principal curve risk is closely related to the unsupervised learning risk, but
imposes restriction on the form of the decoder $\lambda$, i.e., the encoder
$\lambda$ is forced to be the projection index given $g$

The typical approach to estimate principal curves optimizes over $g$ and solver
the non--linear problem of finding $\lambda_g$. To circumvent the non--linear
optimization problem of computing $\lambda_g$ \citet{gerber:iccv09} proposed a
new formulation that switches the optimization over the encoder $\lambda$ while
fixing $g_{\lambda}(z) \equiv E[ X | \lambda(X) = z ]$ to be the conditional
expectation given $\lambda$.  This yields the risk
$d(\lambda, X)^2 =
\frac{1}{2} E \left[\left\lVert g_\lambda( \lambda(X) ) - X
\right\rVert^2\right]$.
Again closely related to the unsupervised least squares risk but with
restrictions imposed on the form of the decoder $g$.

For future reference we term the three different formulations as:
\begin{compactenum}
\item Principal manifold risk (pm--risk): $d(g, X)^2$ with $\lambda_g$
   constrained to be orthogonal to $g$
\item Conditional expectation manifold risk (cem--risk): $d(\lambda, X)^2$ with
   $g_\lambda$ constrained to the conditional expectation given $\lambda$.
\item Unsupervised least squares risk (uls--risk): $d(g, \lambda, X)^2$  with $g$ and
  $\lambda$ unconstrained.
\end{compactenum}
\citet{duchamp:as96} showed that principal curves are saddlepoints of the
pm--risk $d(g, X)^2$ and \citet{gerber2013regularization} showed that critical
points of the cem--risk are weak principal curves. Weak principal curves are,
as the name implies, a slightly weaker version of principal curves:
\begin{definition}[Weak Principal Curves~\citep{gerber2013regularization}]
\label{def:weak_ps}
Let $g:\Lambda \to \mathbf{R}^n$ and $\lambda:\Omega \to \Lambda$. The weak
principal curves of $Y$ are the set $\mathcal{G}_w$ of functions $g$ that
fulfill the self consistency property $E[Y \vert \lambda(Y) = s] = g(s)$ with
$\lambda$ satisfying $\left< y - g(\lambda(y)), \left.
\frac{d}{ds}g(s)\right|_{s=\lambda(y)} \right> = 0 \, \forall y \in \Omega$.
\end{definition}
For principal curves which have no ambiguity points, i.e. all $x \in \Omega$
have a unique closest point on the curve, the definition is equivalent to
principal curves.

\subsection{Critical Points are Principal Manifolds}
\label{sec:critical}
To establish the connection of the unsupervised least squares risk
$$d(g,\lambda, X)^2 = \frac{1}{2}  E\left[ \lVert g\left( \lambda(X) \right) - X
\rVert^2 \right]$$
to principal manifolds, we show that the critical points of
$d(g, \lambda, X)^2$ are weak principal manifolds.

For the pm--risk the optimization is over decoder $g$ only the encoder $\lambda$
is defined in terms of $g$. Vice versa, for the cem--risk the optimization is
over the encoder $\lambda$ only and the decoder $g$ is defined in terms of
$\lambda$. The following theorem establishes that the critical points for the
uls--risk which includes optimization over both $g$ and $\lambda$ are weak
principal manifolds.
\begin{theorem}
\label{th:critical}
The critical points of the unsupervised leasts squares risk
$d(g, \lambda, X)^2 = \frac{1}{2}  E\left[ \lVert g\left( \lambda(X) \right) -
X \rVert^2 \right]$ are weak principal curves.
\end{theorem}
\begin{proof}
The G\^ateaux derivative (or first variation) of $d(\lambda, g,
X)^2$ with respect to $\lambda$ is
\begin{equation}
  \label{eq:ae-glambda}
  \frac{d}{d\epsilon} d(\lambda + \epsilon \tau , g, X)^2 =
  E\left[ \left( g \left( \lambda( X ) - X \right) \right)
  \nabla_g\left(\lambda( X )\right) \tau( X) \right]
\end{equation}
and with respect to $g$
\begin{equation}
  \label{eq:ae-gg}
  \frac{d}{d\epsilon}  d(\lambda, g + \epsilon h, X)^2 = E\left[ \langle
                               g\left( \lambda( X ) \right) - X , h\left(
                           \lambda(X) \right) \rangle \right] \, .
\end{equation}
At a critical point these have to be pointwise zero for any $\tau$ (variation
of $\lambda$) and $h$ (variation of $g$), respectively. This yields the
conditions $\left( g \left( \lambda( X ) - X \right)
\right)\nabla_g\left(\lambda( X )\right) = 0$ from
equation~\eqref{eq:ae-glambda} and $g(z) =  E[ X | \{x :\lambda(x) = z\} ] $
for any $z$ in the image of $\lambda$ from equation~\eqref{eq:ae-gg}. This
establishes that critical points of the uls--risk are weak principal manifolds.
\end{proof}

The critical points contain uninteresting minimal solutions with zero
reconstruction residual, i.e.,  space filling manifolds and the identity
mapping. Critical points of the pm--risk and cem--risk are saddlepoints, since
the uls--risk is more flexible, i.e., is a superset of both the pm--risk and
the cem-risk, the critical points of the uls--risk are saddlepoints as well.

The uls--risk poses additional challenges. If $\lambda$ is permitted to be
injective, i.e., no dimensionality constraints, then the condition that $g(z)
= E[ X | \{x :\lambda(x) = z\} ] $ can only be satisfied by the identity
mapping, since the set $\{x :\lambda(x) = z\}$ is a single point. Additionally,
in that setting the critical points have a discontinuous G\^ateaux derivative:
Consider a critical point with encoder $\lambda_*$ and decoder $g_*$, with rank
$\nabla_\lambda$ less than $n$. Let $\pi_\alpha$ be a set of maps such that
$\lambda + \pi_\alpha$ is injective and $\lVert \pi_\alpha(x) \rVert \leq
\alpha$. Now from the critical point conditions $\int_{\lambda^{-1}(\{z\})}
\left(g(z) - x \right) p(x) dx = 0$. Let $s = (\lambda +
\pi_\alpha)^{-1}(\{z\})$, now for any $\alpha > 0 $ we have $\int_{(\lambda +
\pi_\alpha)^{-1}(\{z\})} \left(g(z) - x \right) p(x) dx = \left( g(s) - x
\right) p(x) + \alpha \lVert \nabla_g(s) \rVert + O(\alpha^2) $ which cannot be
made arbitrarily small. This discontinuity is expected since the expectation
$E$ changes abruptly when moving to an injective function, or in fact at any
change in the rank of the Jacobian $\nabla_\lambda$. Thus, in order to avoid
the identity mapping it is necessary to constrain the dimensionality of the
reconstruction function.

\section{Shaping the Unsupervised Least Squares Risk}
\label{sec:saddles}

To avoid fitting overly curved solutions~\citet{rifai2011higher} propose an
explicit penalty on the shape of the solution by adding a penalty on the
Hessian to contractive encoding.  This approach combines a regularization on
the dimensionality, the Jacobian, and a regularization on the curvature, the
Hessian.  Penalizing the Hessian is akin to the proposal
by~\citet{kegl:tpami00} in the context of fitting principal manifolds.
However, the saddlepoint nature of the objective function makes it infeasible
to use cross--validation for tuning the amount of regularization required.

To address the saddlepoint challenge we propose to change the objective
function such that all critical points are local minima.
\citet{gerber2013regularization} applied this approach to principal manifolds.
They observed that principal curves are saddlepoints of the risk because curves
with smaller risk can be achieved by either violating the conditional
expectation constraint in the pm--risk or by violating the orthogonality
constraint in the cem--risk formulation.  This
lead~\citet{gerber2013regularization} to minimize orthogonality using the risk:
\begin{equation}
\label{eq:ortho}
q(\lambda, X)^2 =
E \left[ \left\langle
\left( g_\lambda( \lambda(X) ) - X \right)
,
\left.  \frac{d}{dz} g(z) \right|_{z=\lambda( X )}
\right\rangle^2 \right]  = 0
\, .
\end{equation}
\citet{gerber2013regularization} showed that all critical points of this
risk are minima and principal curves.

In this section we show that there is a general principle underlying the
derivation by~\citet{gerber2013regularization} that is not restricted to the
principal manifold case. In Section~\ref{sec:gnorm} we derive the general
principle from Newton's method for optimization and show how it leads to the
orthogonal risk in the case of principal manifolds and in
Section~\ref{sec:ortho-encoding} we apply it to auto--encoding which results in an
orthogonal contractive penalty.

\subsection{Gradient--Norm Minimization}
\label{sec:gnorm}
Newton's method for finding critical points of a function $f : \mathbb{R}^m \to
\mathbb{R}$ is to find the zero crossings of $\nabla f$. Newton's method
applied to the gradient of a multivariate function $f : \mathbb{R}^m \to
\mathbb{R}$ yields updates of the form
$$
\mathbf{x}^{k+1} = \mathbf{x}^k - \alpha \mathbf{v}
$$
with $\mathbf{v}$ a solution to
$$
\nabla^2 f( \mathbf{x}^k) \mathbf{v} =  \nabla f(\mathbf{x}^k)
$$
where $\nabla^2 f$ is the Hessian of $f$. For $\nabla^2 f$ indefinite the
iterations moves---given an appropriate step size---towards a saddlepoint.

If the Newton step is difficult to solve, i.e., the inversion of the Hessian to
expensive, one can resort to minimizing the gradient norm $||\nabla f(x)||^2$.
Since $||\nabla f(x)||^2 \ge 0$, the critical points of $f$ are minima of the
gradient--norm risk. However, depending on the structure of $f$, additional
critical points are possible.  At critical points of $||\nabla f(x)||^2$ the
gradient $\nabla^2 f(x) \nabla f(x)$ has to be zero. Thus, either the gradient
$\nabla f$ has to be zero, the Hessian $\nabla^2 f$ is zero or the gradient is
a linear combination of directions with zero curvature, i.e. $\nabla f v = 0$
or $\nabla^2 f v = 0$ for all directions $v$. For functions with non-degenerate
Hessian all critical points require $\nabla f = 0$, and the gradient--norm
minimization finds a critical point of $f$. For functions with degenerate
Hessians, inflection points in the direction of the gradient are additional
critical points of $|| \nabla f ||^2$.  Such critical points can pose a problem
for finding critical points of $f$ through gradient--norm minimization and need
to be evaluated.

Newton's method is a scaled steepest descent method with scaling by
$\nabla^2 f(x)^{-1}$.  The gradient of $\nabla f(x)$ is $\nabla^2 f(x) \nabla
f(x)$, i.e., a steepest descent with scaling by $\nabla^2 f(x)$. This avoids
having to invert the Hessian but worsen the condition number and results
in slower convergence~\citep{boyd2004convex}.

\subsection{Gradient--Norm Minimization for Principal Manifolds}
Applying the gradient-norm minimization procedure to the cem--risk recovers the
geometrically derived formulation by~\citet{gerber2013regularization}. The
gradient of $d(\lambda, X)^2$ with respect to $\lambda$ is:
\begin{equation}
E \left[ \left\langle \left( g_\lambda( \lambda(X)) - X \right) ,
         \left. \frac{d}{ds} g(s) \right|_{s=\lambda( X )} \right\rangle \right]
\, .
\end{equation}
By the calculus of variations this has to hold pointwise. Applying the
gradient--norm minimization pointwise recovers the orthogonality risk in
Equation~\ref{eq:ortho}.

\subsection{Gradient--Norm Minimization for Auto--encoding}
\label{sec:ortho-encoding}
Applying the gradient-norm minimization to the uls--risk from
Equations~\eqref{eq:ae-glambda} and~\eqref{eq:ae-gg} results in the risk:
\begin{equation}
  E\left[ \left( g\left( \lambda(X)\right) - X \right)^2 \right] + \alpha
  E\left[ \left( g\left( \lambda(X) \right)- X \right)^2 \nabla_g\left( \lambda( X) \right)^2 \right]
\end{equation}
with $\alpha=1$.
The first term is the squared residual and the second term can be seen as a
directional contraction penalty; the Jacobian of $g$ is penalized but only in
the direction of the residual vector. For principal manifolds the second term
is zero since the Jacobian of $g$ has to be zero in the direction of the
residual, i.e. $g \cdot \lambda$ is an orthogonal projection to $g$. Treating
the second term of the derivative as a penalty we can weigh it differently
through changing $\alpha$ during optimization.

Figure~\ref{fig:bottleneck-penalty} illustrates this approach on a bottleneck
neural network architecture and compares it with no and contractive
regularization.
\begin{figure}[htb] \centering
\begin{tabular}{ccc}
  \includegraphics[width=0.3\linewidth]{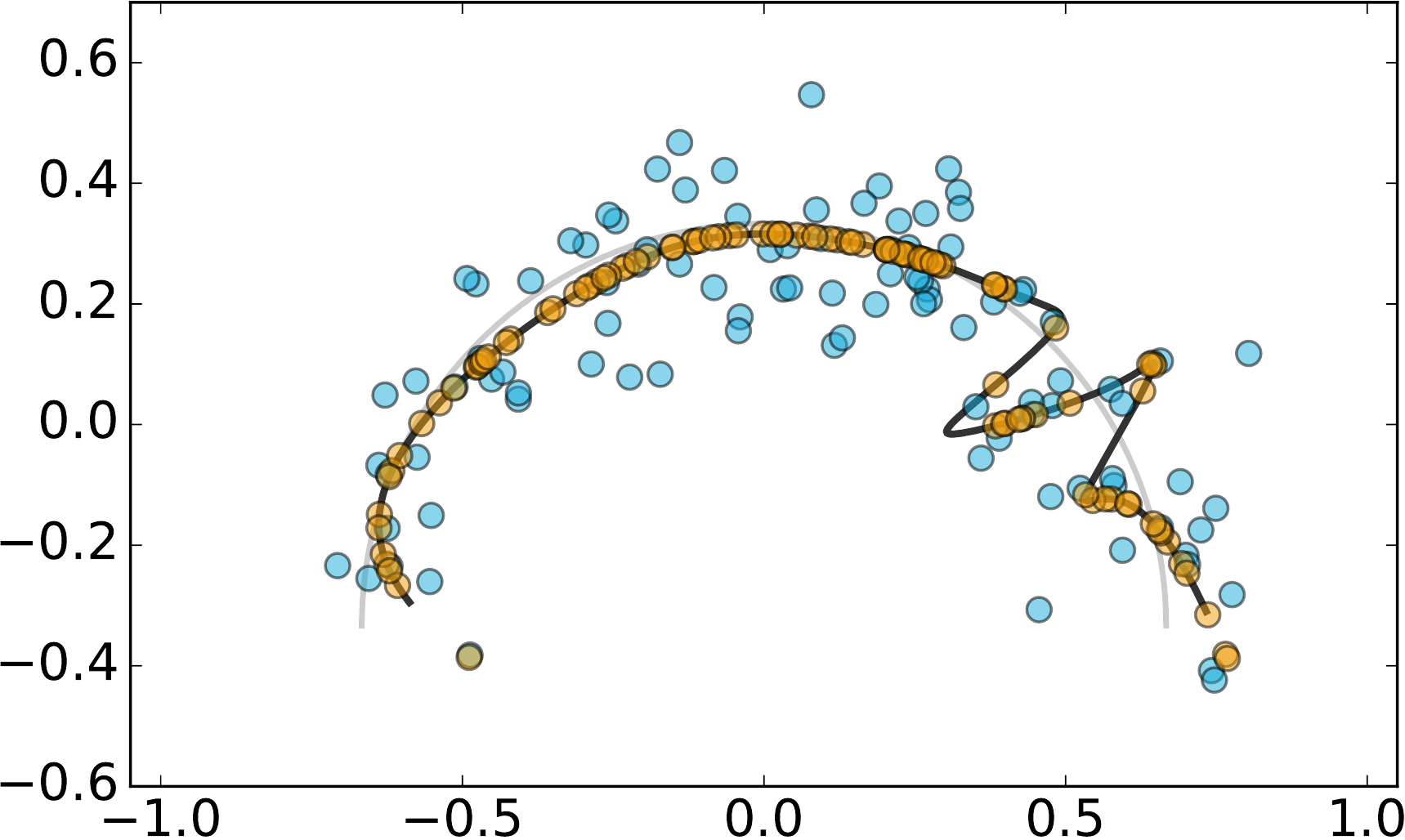} &
  \includegraphics[width=0.3\linewidth]{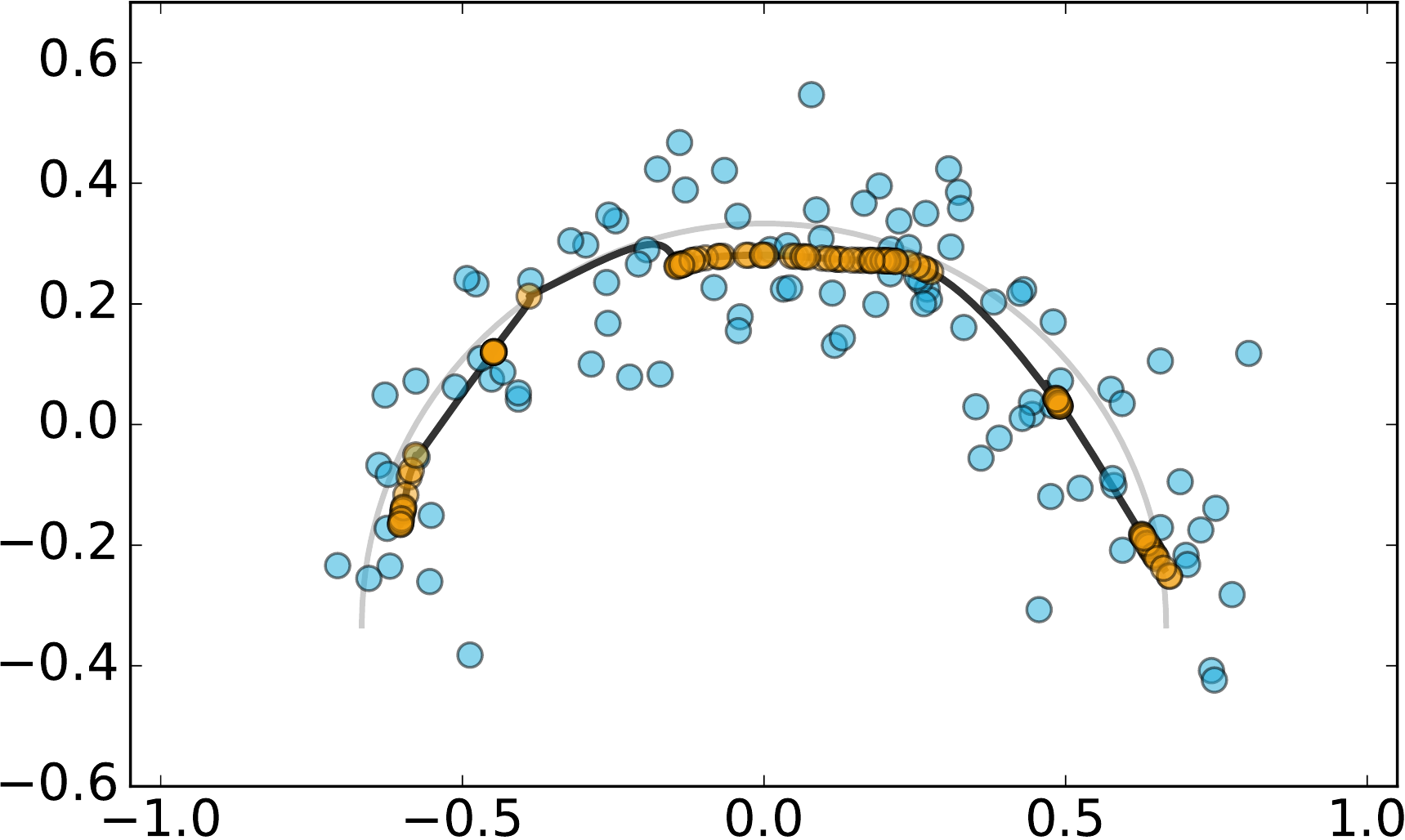} &
  \includegraphics[width=0.3\linewidth]{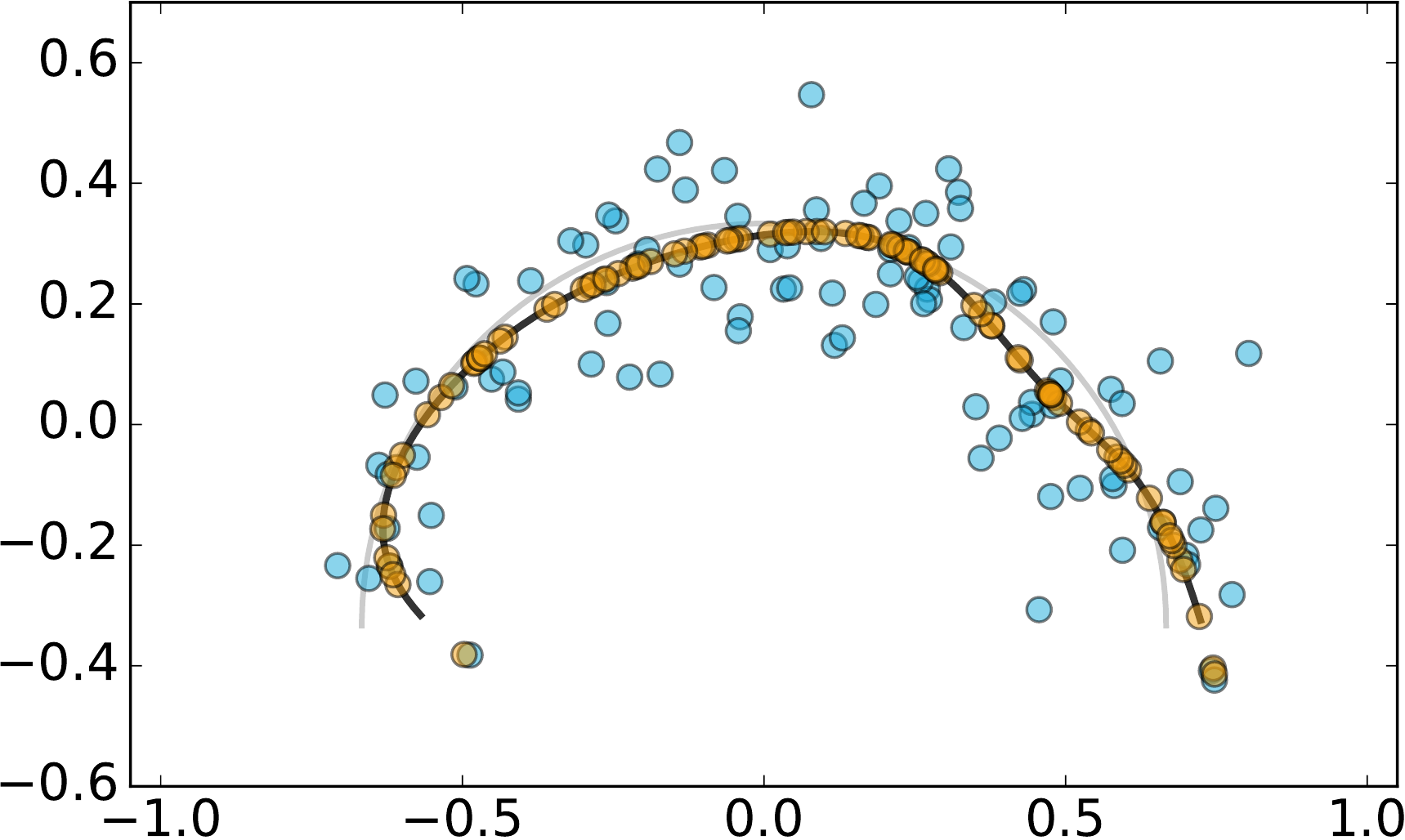} \\
  (a) & (b) & (c)
\end{tabular}
\caption{The image show the minima found after 20,000 iterations using a
  network architecture with hidden layers of
  50--100--200--1--200--100--50, units and (a) no regularization (b)
  contraction penalty and (c) orthogonal contractive penalty.
  \label{fig:bottleneck-penalty}
}
\vspace{-0.05in}
\end{figure}

\begin{figure}[htb]
\centering
  \includegraphics[width=0.9\linewidth]{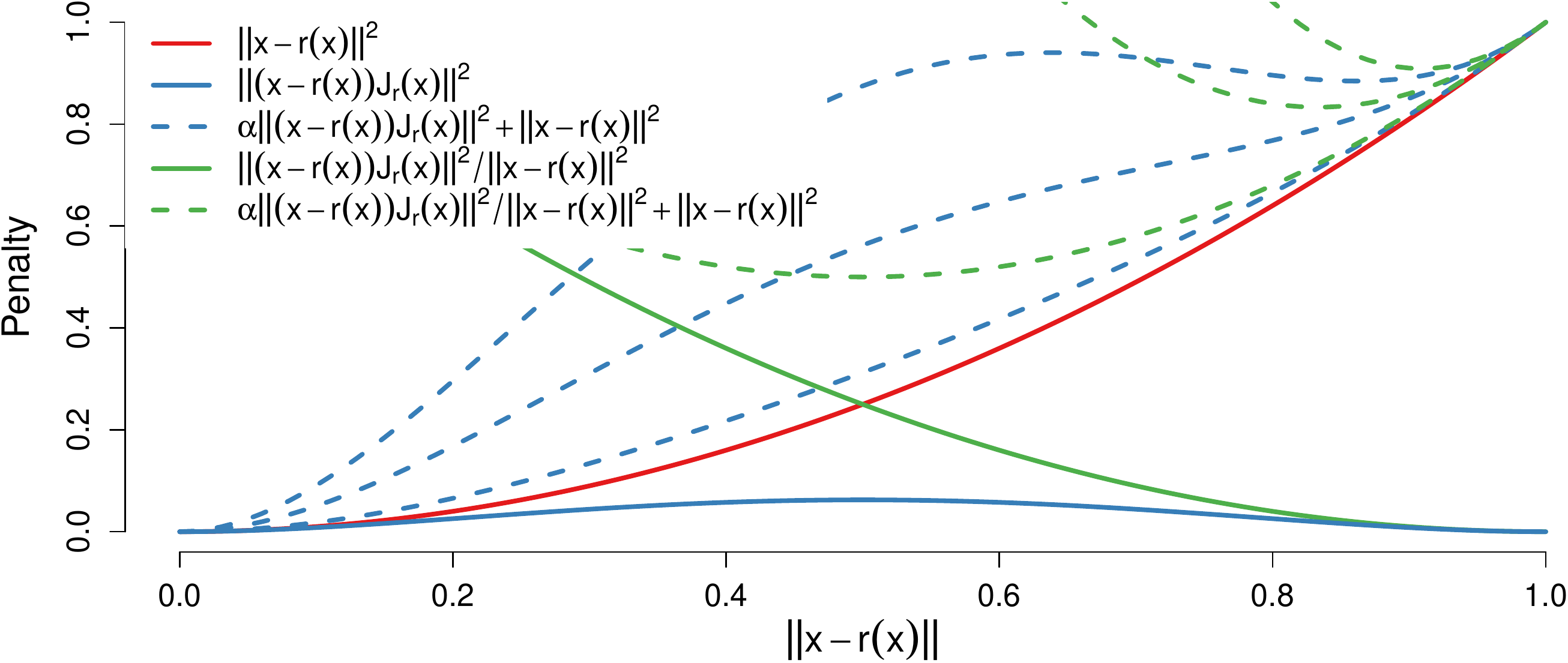}
\caption{The effect of the orthogonal contractive penalty on an idealized
  example problem. The idealized setting assumes that the Jacobian norm is
  perfectly anti--correlated with the residual. That is, as the reconstruction
  functions moves away from the identity function (at $x-r(x) = 0$) towards a
  principal manifold (at $x-r(x) = 1$) the increase in error is matched by a
  decrease in the Jacobian. The orthogonal contraction has a minima at both 0
  and 1, but combined with the squared residual term, the effect is
  negligible. Increasing the orthogonal penalty creates a minima closer to the
  principal manifold, however, the identify function still is minimal and has a
  larger region of attraction.  The normalized orthogonal penalty actively
  pushes the solution away from the identity solution and combined with the
  residual term leads to a minima close to the principal manifold.
\label{fig:penalty-shapes} }
\vspace{-0.05in}
\end{figure}
The method for finding saddlepoints does not address the issue of discontinuous
saddlepoints if the dimensionality is not restricted. The orthogonal
contractive penalty does typically not help to find such degenerate solutions
and results in the identity solution.  This behaviour is expected since moving
towards the identity solution often also reduces the orthogonal contractive
penalty as illustrated in an idealized setting in Figure~\ref{fig:penalty-shapes}.
To remedy this problem we consider normalizing the contractive penalty to:
\begin{equation}
  \alpha E\left[ \frac{\left( g \left( \lambda(X) \right) - X \right)^2 \nabla_g\left( \lambda( X) \right)^2}
 {\left( g\left( \lambda(X) \right) - X \right)^2}  \right] \, .
\end{equation}
This has the effect that the penalty increases quadratically towards the
identity solutions, as illustrated in Figure~\ref{fig:penalty-shapes}. The
normalized penalty counterweights the quadratic decrease in the reconstruction
error. Figure~\ref{fig:penalty-orthogonal} demonstrates the desired effect of
pushing the solution away from the identity.
\begin{figure}[bht] \centering
\begin{tabular}{ccc}
  \includegraphics[width=0.3\linewidth]{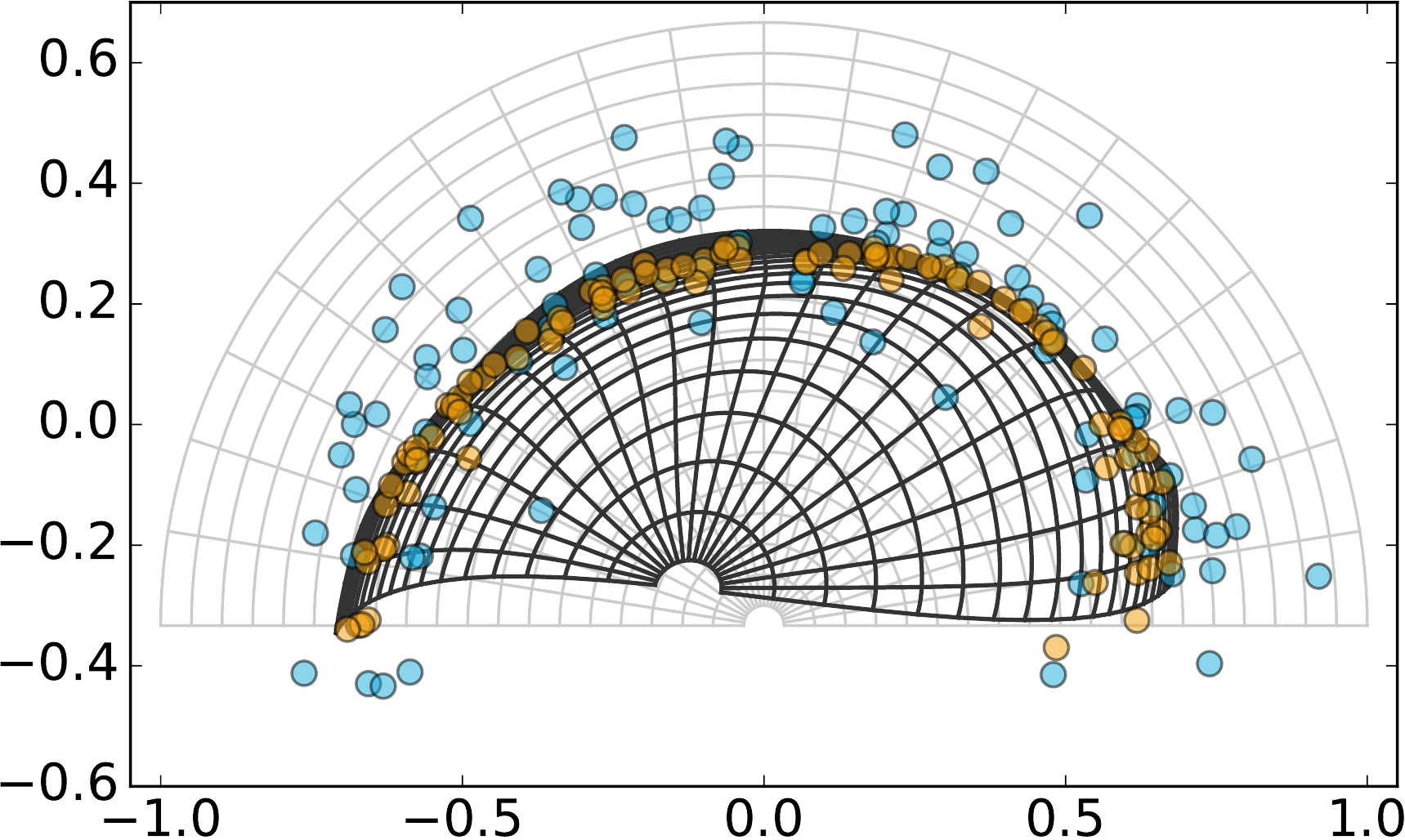} &
  \includegraphics[width=0.3\linewidth]{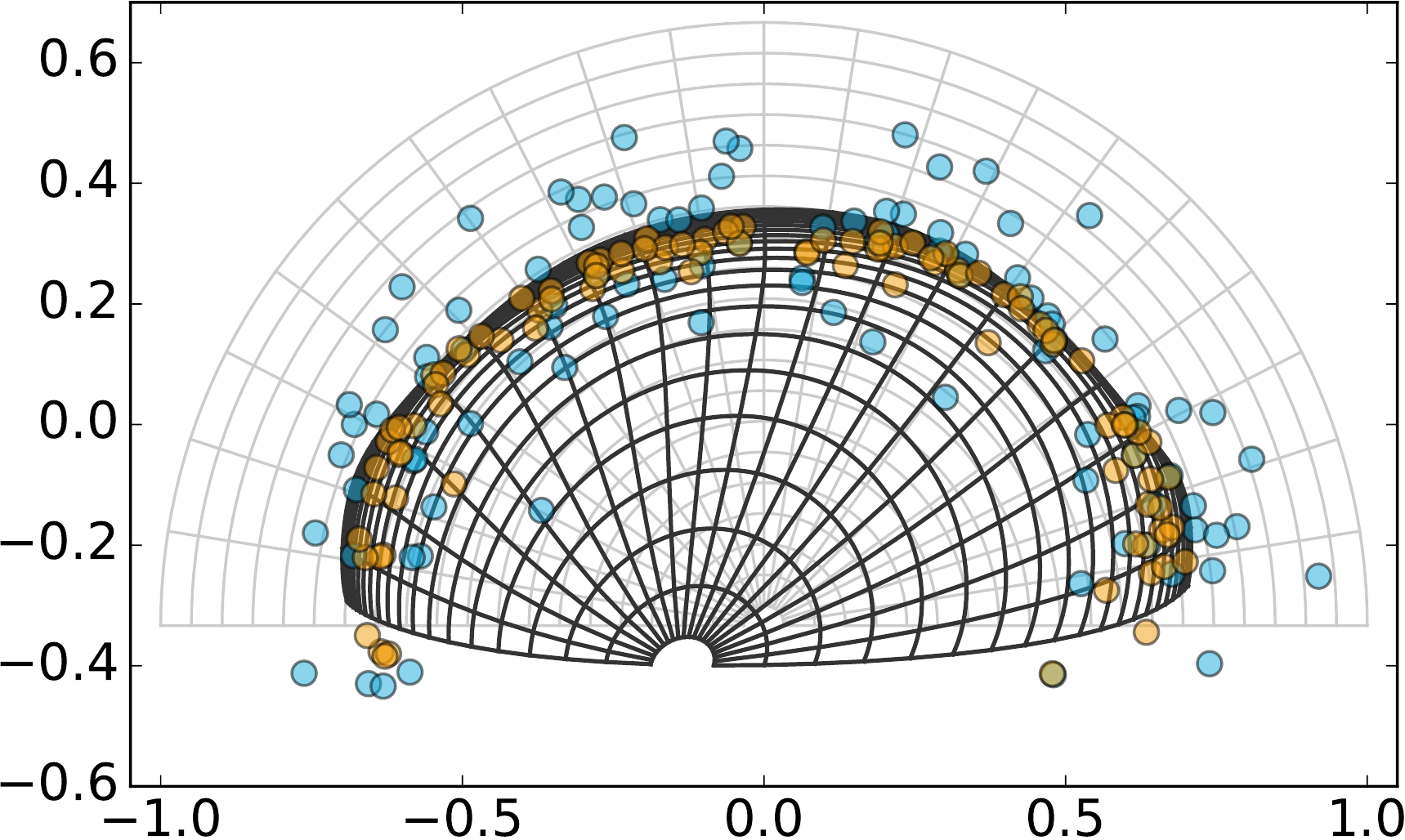} &
  \includegraphics[width=0.3\linewidth]{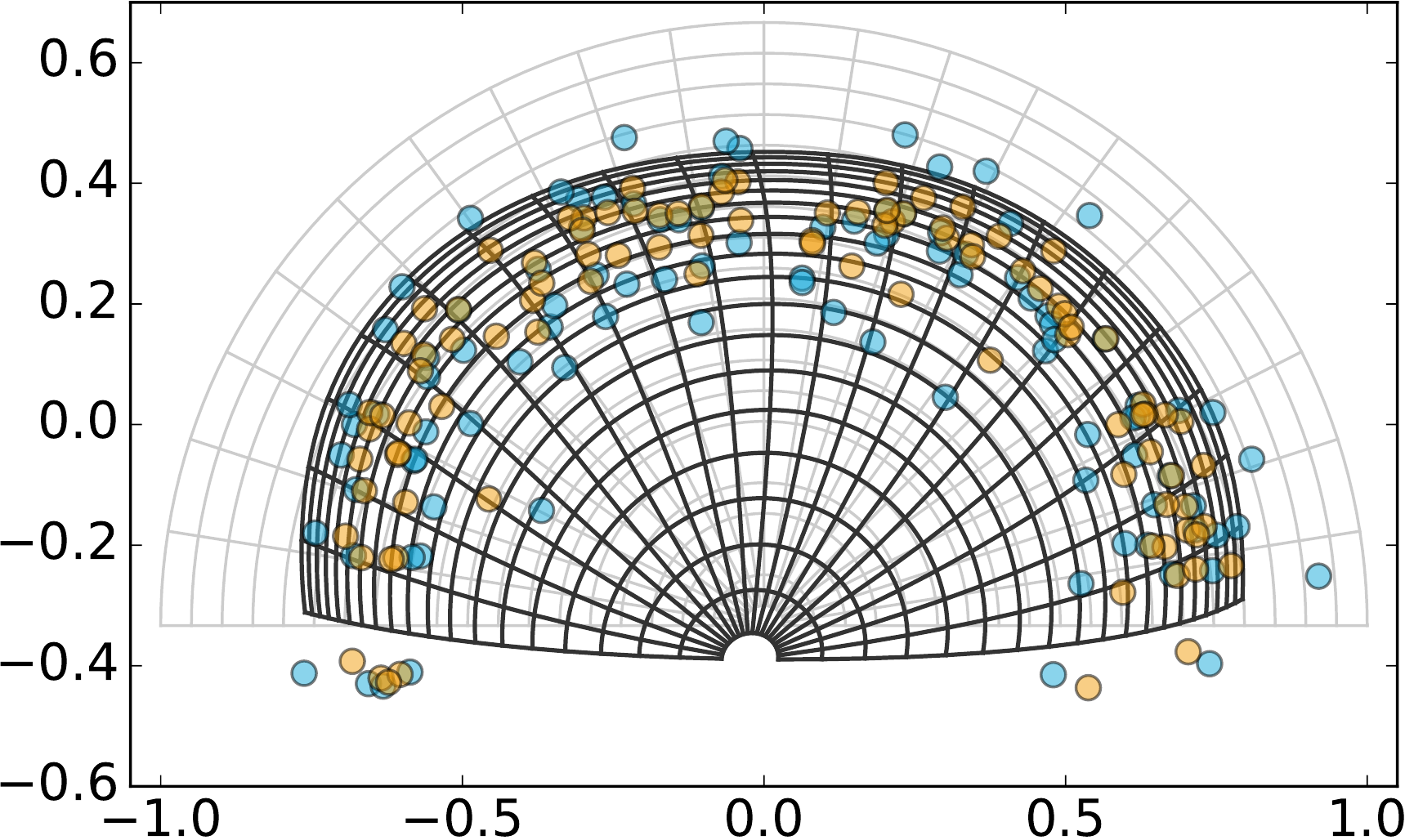} \\
  (a) & (b) & (c)
\end{tabular}
\caption{The image show the minima found after 20'000 iterations using a
  network architecture with hidden layers of 50--100--200--100--50, units and
  (a) 0.04 (b) 0.02 (c) 0.005  weighted orthogonal penalty. The orthogonal
  penalty yields reasonable results for a range of penalty weights.
  \label{fig:penalty-orthogonal}
}
\vspace{-0.05in}
\end{figure}

The orthogonal penalty works well for examples with a co--dimension of one. For
more realistic examples, with higher co--dimension, the orthogonal penalty
reaches zero if the residual vectors are contained in any subspace of the
normal bundle. In this case the orthogonal contractive penalty has to be
combined with a dimensionality constraint. Figure~\ref{fig:denoise-ortho} shows
the effect on denoising auto--encoding. The orthogonal contractive penalty
preferentially selects directions that tend towards an orthogonal projection.
\begin{figure}[htb]
\centering
\begin{tabular}{ccc}
  \includegraphics[width=0.23\linewidth]{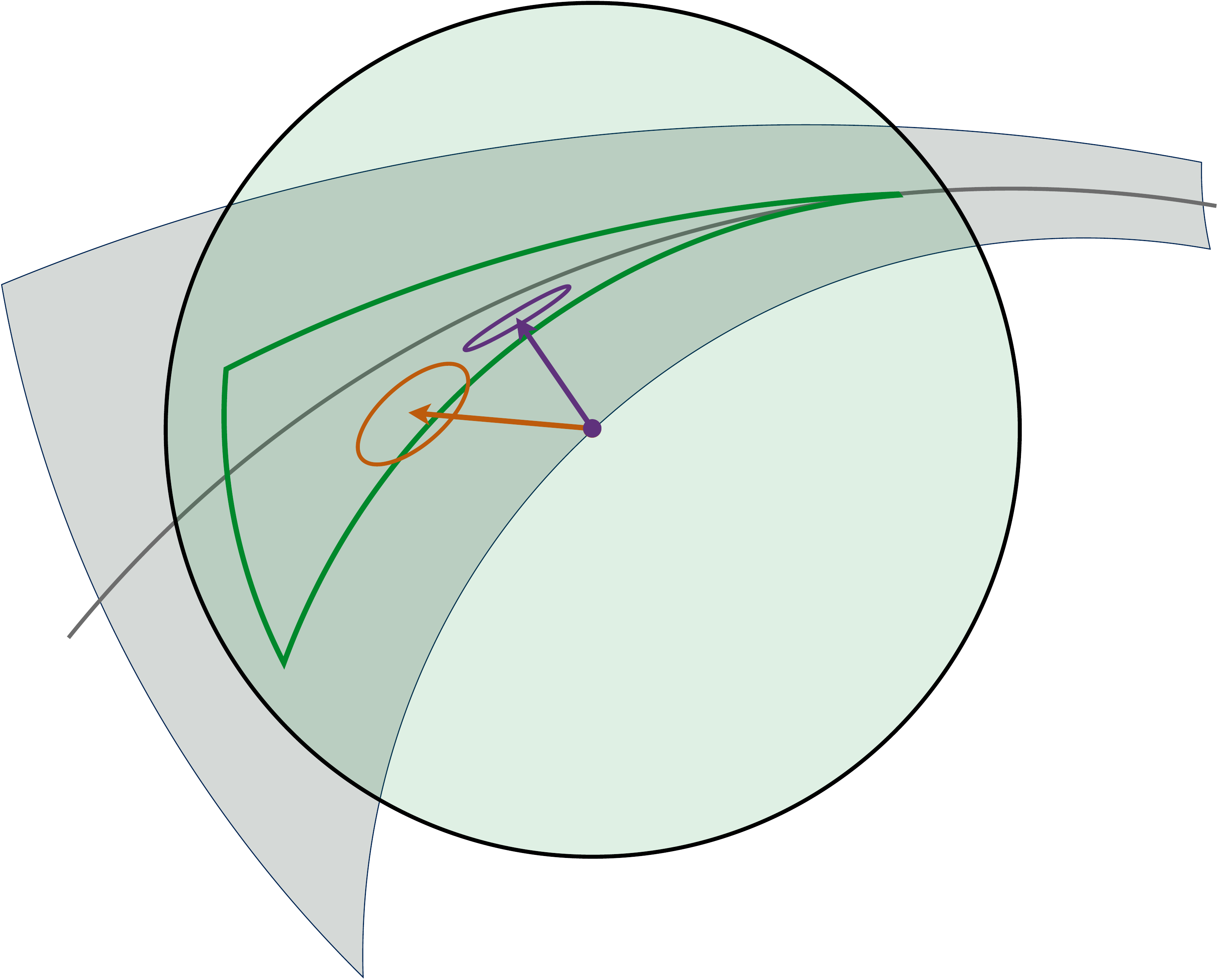} &
  \includegraphics[width=0.35\linewidth]{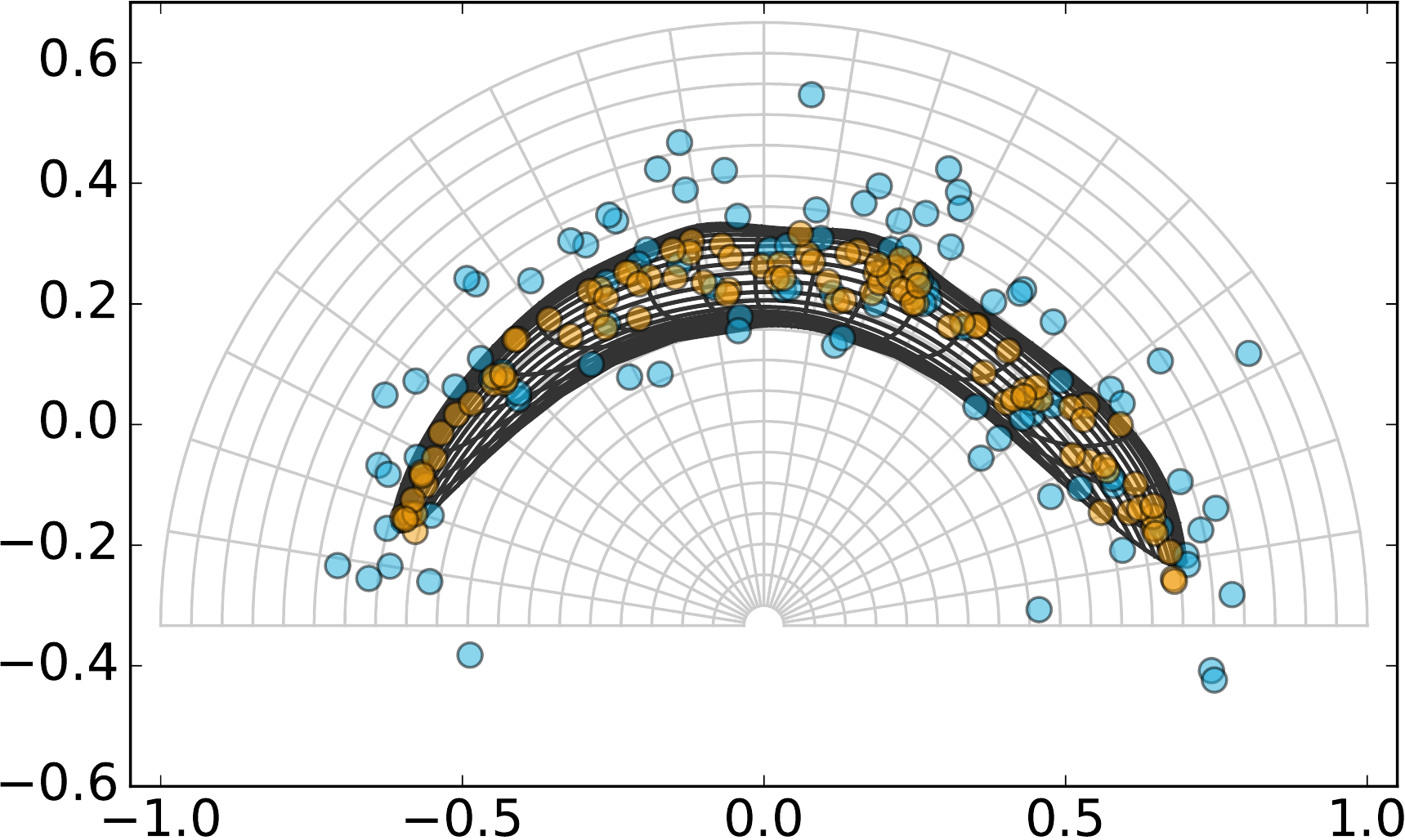} &
  \includegraphics[width=0.35\linewidth]{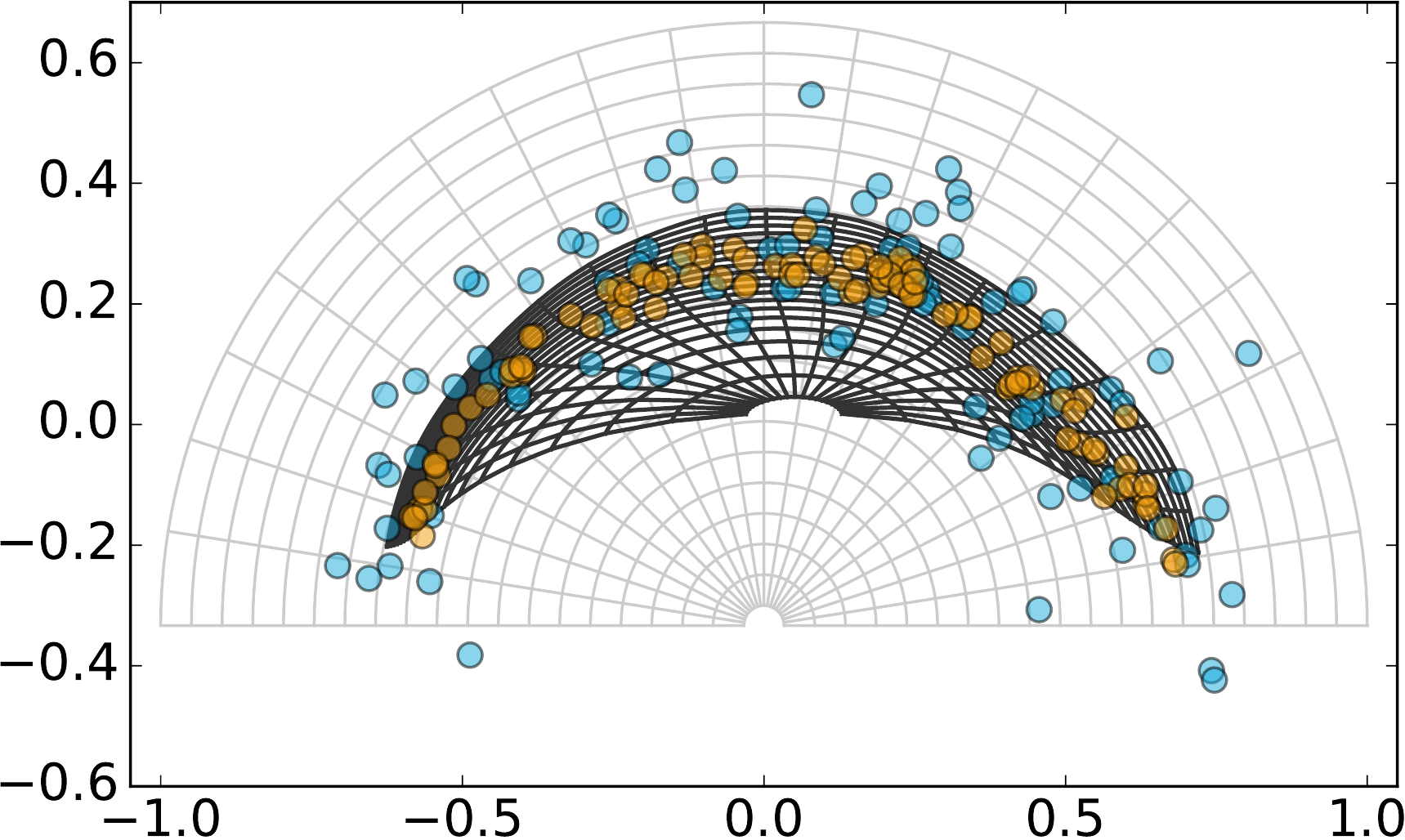} \\
  (a) & (b) & (c)
\end{tabular}
\caption{The (a) illustration depicts the effect of orthogonal contraction with
  denoising auto--encoding.  In the limit, points from the density (gray area)
  with noise added, $x + \epsilon$ (black dot) are preferentially mapped
  orthogonal to the contraction (purple) as compared to denoising alone
  (orange). The (b) orthogonal contraction  with denoising leads to sharper
  boundary as compared to (c) denoising alone.  This indicates that orthogonal
  contraction with densoing can yield solutions that generalize better to
  unseen data than denosing alone.
  \label{fig:denoise-ortho}
}
\vspace{-0.05in}
\end{figure}

\section{Conclusion}
The connection of auto--encoding neural networks with principal manifolds casts
regularization as an approach to find saddlepoints of the mean squared error
risk. The issue of saddlepoints has been recently explored in the neural
network literature. \citet{pascanu2014saddle} develop an optimization strategy
to avoid saddlepoints, while \citet{choromanska2015loss} argue that
saddlepoints are desirable solutions in deep learning. The connection to
principal manifolds makes the argument for saddlepoints as desirable solutions
explicit in the case of auto--encoding networks.

We proposed a new method, gradient--norm minimization for finding saddlepoints.
The gradient--norm minimization is potentially converging slowly. Both Newton's
method and gradient--norm minimization are gradient descent schemes under
scaling.  Newton's method scales by the inverse Hessian, while the proposed
method corresponds to a scaling by the Hessian. This scaling increases the
condition number and decreases the convergence rate.

\bibliographystyle{abbrvnat}
\bibliography{sgerber,pcae}

\begin{thebibliography}{12}
\providecommand{\natexlab}[1]{#1}
\providecommand{\url}[1]{\texttt{#1}}
\expandafter\ifx\csname urlstyle\endcsname\relax
  \providecommand{\doi}[1]{doi: #1}\else
  \providecommand{\doi}{doi: \begingroup \urlstyle{rm}\Url}\fi

\bibitem[Alain and Bengio(2014)]{alain2014regularized}
G.~Alain and Y.~Bengio.
\newblock What regularized auto-encoders learn from the data-generating
  distribution.
\newblock \emph{The Journal of Machine Learning Research}, 15\penalty0
  (1):\penalty0 3563--3593, 2014.

\bibitem[Boyd and Vandenberghe(2004)]{boyd2004convex}
S.~Boyd and L.~Vandenberghe.
\newblock \emph{Convex optimization}.
\newblock Cambridge university press, 2004.

\bibitem[Choromanska et~al.(2015)Choromanska, Henaff, Mathieu, Ben~Arous, and
  LeCun]{choromanska2015loss}
A.~Choromanska, M.~Henaff, M.~Mathieu, G.~Ben~Arous, and Y.~LeCun.
\newblock The loss surface of multilayer networks.
\newblock In \emph{AI \& Statistics (AIStats 2015)}. arXiv:1412.0233, 2015.

\bibitem[Duchamp and Stuetzle(1996)]{duchamp:as96}
T.~Duchamp and W.~Stuetzle.
\newblock Extremal properties of principal curves in the plane.
\newblock \emph{The Annals of Statistics}, 24\penalty0 (4):\penalty0
  1511--1520, 1996.

\bibitem[Gerber and Whitaker(2013)]{gerber2013regularization}
S.~Gerber and R.~Whitaker.
\newblock Regularization-free principal curve estimation.
\newblock \emph{The Journal of Machine Learning Research}, 14\penalty0
  (1):\penalty0 1285--1302, 2013.

\bibitem[Gerber et~al.(2009)Gerber, Tasdizen, and Whitaker]{gerber:iccv09}
S.~Gerber, T.~Tasdizen, and R.~Whitaker.
\newblock Dimensionality reduction and principal surfaces via kernel map
  manifolds.
\newblock In \emph{IEEE 12th International Conference on Computer Vision},
  pages 529--536, 2009.

\bibitem[Hastie and Stuetzle(1989)]{hastie:jasa89}
T.~Hastie and W.~Stuetzle.
\newblock Principal curves.
\newblock \emph{Journal of the American Statistical Association}, 84\penalty0
  (406):\penalty0 502--516, 1989.

\bibitem[Kawaguchi(2016)]{kawaguchi2016deep}
K.~Kawaguchi.
\newblock Deep learning without poor local minima.
\newblock In \emph{Advances in Neural Information Processing Systems}, pages
  586--594, 2016.

\bibitem[K\'{e}gl et~al.(2000)K\'{e}gl, Krzyzak, Linder, and
  Zeger]{kegl:tpami00}
B.~K\'{e}gl, A.~Krzyzak, T.~Linder, and K.~Zeger.
\newblock Learning and design of principal curves.
\newblock \emph{IEEE Transaction On Pattern Analysis Machine Intelligence},
  22\penalty0 (3):\penalty0 281--297, 2000.

\bibitem[Pascanu et~al.(2014)Pascanu, Dauphin, Ganguli, and
  Bengio]{pascanu2014saddle}
R.~Pascanu, Y.~N. Dauphin, S.~Ganguli, and Y.~Bengio.
\newblock On the saddle point problem for non-convex optimization.
\newblock \emph{arXiv preprint arXiv:1405.4604}, 2014.

\bibitem[Rifai et~al.(2011)Rifai, Mesnil, Vincent, Muller, Bengio, Dauphin, and
  Glorot]{rifai2011higher}
S.~Rifai, G.~Mesnil, P.~Vincent, X.~Muller, Y.~Bengio, Y.~Dauphin, and
  X.~Glorot.
\newblock Higher order contractive auto-encoder.
\newblock \emph{Machine Learning and Knowledge Discovery in Databases}, pages
  645--660, 2011.

\bibitem[Vincent et~al.(2008)Vincent, Larochelle, Bengio, and
  Manzagol]{vincent2008extracting}
P.~Vincent, H.~Larochelle, Y.~Bengio, and P.-A. Manzagol.
\newblock Extracting and composing robust features with denoising autoencoders.
\newblock In \emph{Proceedings of the 25th international conference on Machine
  learning}, pages 1096--1103. ACM, 2008.

\end{thebibliography}

\end{document}